\documentclass{article}

\PassOptionsToPackage{numbers, compress}{natbib}

\usepackage{enumitem}

\usepackage[preprint]{neurips_2022}

\usepackage[utf8]{inputenc} 
\usepackage[T1]{fontenc}    
\usepackage{hyperref}       
\usepackage{url}            
\usepackage{booktabs}       
\usepackage{amsfonts}       
\usepackage{nicefrac}       
\usepackage{microtype}      
\usepackage{xcolor}         

\usepackage{amssymb}
\usepackage{mathtools}
\usepackage{amsthm}
\usepackage{amsmath}\allowdisplaybreaks 

\usepackage[capitalize,noabbrev]{cleveref}

\usepackage[textsize=tiny]{todonotes}

\usepackage{algorithm}
\usepackage{algorithmic}

\usepackage{microtype}
\usepackage{graphicx}
\usepackage{subfigure}
\usepackage{booktabs} 
\usepackage{wrapfig}

\theoremstyle{plain}
\newtheorem{theorem}{Theorem}[section]
\newtheorem{proposition}[theorem]{Proposition}
\newtheorem{lemma}[theorem]{Lemma}
\newtheorem{corollary}[theorem]{Corollary}
\theoremstyle{definition}
\newtheorem{definition}[theorem]{Definition}
\newtheorem{assumption}[theorem]{Assumption}
\theoremstyle{remark}
\newtheorem{remark}[theorem]{Remark}

\allowdisplaybreaks

\title{Learning to Control under Time-Varying Environment}

\author{%
  Yuzhen Han\\
  Department of Mechanical \& Industrial Engineering\\
  University of Toronto\\
  Toronto, ON CANADA \\
  \And
  Ruben Solozabal \\
  MBZUAI \\
  Masdar City, Abu Dhabi, UAE \\
  \texttt{ruben.solozabal@mbzuai.ac.ae} \\
  \AND
  Jing Dong \\
  The Chinese University of Hong Kong \\
  Shenzhen, China \\
  \texttt{jingdong@link.cuhk.edu.cn} \\
  \And
  Xingyu Zhou \\
  Wayne State University \\
  Detroit, USA \\
  \texttt{xingyu.zhou@wayne.edu} \\
  \And
  Martin Tak\'a\v{c} \\
  MBZUAI \\
  Masdar City, Abu Dhabi, UAE  \\
  \texttt{Takac.MT@gmail.com} \\
  \And
  Bin Gu \\
  MBZUAI \\
  Masdar City, Abu Dhabi, UAE  \\
  \texttt{bin.gu@mbzuai.ac.ae} \\
}

\begin{document}

\setlength{\abovedisplayskip}{3pt}
\setlength{\belowdisplayskip}{3pt}
\maketitle

\begin{abstract}
  This paper investigates the problem of regret minimization in linear time-varying (LTV) dynamical systems. Due to the simultaneous presence of uncertainty and non-stationarity, designing online control algorithms for unknown LTV systems remains a challenging task. At a cost of NP-hard offline planning, prior works have introduced online convex optimization algorithms, although they suffer from nonparametric rate of regret.
  In this paper, we propose the first computationally tractable online algorithm with regret guarantees that avoids offline planning over the state linear feedback policies. Our algorithm is based on the optimism in the face of uncertainty (OFU) principle in which we optimistically select the best model in a high confidence region. Our algorithm is then more explorative when compared to previous approaches. To overcome non-stationarity, we propose either a restarting strategy (R-OFU) or a sliding window (SW-OFU) strategy. With proper configuration, our algorithm is attains sublinear regret $O(T^{2/3})$. These algorithms utilize data from the current phase for tracking variations on the system dynamics. We corroborate our theoretical findings with numerical experiments, which highlight the effectiveness of our methods. To the best of our knowledge, our study establishes the first model-based online algorithm with regret guarantees under LTV dynamical systems.

\end{abstract}

\section{Introduction}

Regret minimization in online control has been extensively investigated in the context of either unknown time-invariant or known time-varying dynamics systems. Yet real applications such as dynamic pricing and ad allocations call for the need for an unknown time-varying system. Under such a setting, the problems become significantly more challenging due to the coexistence of non-stationarity and uncertainty. Despite previous attempts on unknown LTV on stable controllers \cite{middleton1988adaptive} or system identification \cite{sarkar2019nonparametric}, it remains open whether an algorithm can achieve meaningful regret guarantees in this scenario. This paper thus addresses the problem of minimizing the cumulative regret in unknown LTV systems
\begin{equation}
{{\mathcal{R}}_{T}}=\textstyle{\sum}_{t=1}^{T}{\left( \min_a f(a)-f({{a}_{t}}) \right)} \,,
\label{I1}
\end{equation}
where $f$ is the (cost) function being optimized, and $a_t$ is the action chosen at time $t$. To the best of our knowledge, this is the first work that achieves regret guarantees with a computationally tractable algorithm.

When the system is time-invariant, the regret minimization~(\ref{I1}) problem has been well studied. With offline simulations, numerous existing results achieve sublinear regret (i.e. $O(\sqrt{T})$) \cite{cassel2021online,agarwal2019online,cohen2019learning,simchowitz2020naive,ManiaTR19,abeille2020efficient}. By further encouraging exploration with intrinsic noise from the system dynamics, \cite{cassel2020logarithmic,foster2020logarithmic} achieve a logarithmic regret of $O(\log^2T)$. Recent work\cite{chen2021black} presents a finite-time sublinear regret from a single chain of black-box interactions without access to offline simulations. With online convex optimization (OCO) and structured memory \cite{shi2020online} achieves a constant, dimension-free competitive ratio of regret.

The regret minimization problem is complicated when the system is time-varying. This encapsulates a wide range of possible scenarios, the dynamics can be slowly changing or abruptly changing. With the knowledge of the system dynamic, several approaches that have investigated~(\ref{I1}) in LTV systems, summarized in Table 1. Specifically, \cite{lin2021perturbation} studies regret for predictive control of LTV systems. The iGPC~\cite{agarwal2021regret} utilizes a nested-OCO formulation to design an iterative algorithm for minimizing planning regret in the presence of a time-varying system. Similarly, \cite{gradu2020adaptive} adopts OCO with memory to minimize the adaptive regret, which is the supremum of the local regret (with respect to the local optimal comparator) over all contiguous intervals in time. 

While promising results are presented in LTV with known system dynamics, such requirement is often too stringent, if not impossible to fulfill in real applications. Yet the uncertainty of the system dynamic poses new challenges for algorithm designs and regret guarantees. To the best of our knowledge, there is only one work that addresses~(\ref{I1}) on unknown LTV environments~\cite{minasyan2021online}. The work achieves a sublinear regret bound for convex parametrization policies. However, for the class of linear state policies ($u=Kx$), the regret is proportional to $\exp(\Omega  (nm)$, where $n$, $m$ are the dimension of state and action spaces. This reveals the impractical nature of the algorithm as it can be intractable for a wide range of problems. Furthermore, the algorithms rely on an offline planning procedure over the entire state linear feedback policies. While this is possible in a linear time-invariant system, which admits efficient convex relaxations, this is NP-hard in LTV with unknown dynamics. The following question remains open.

\textit{Does there exist a computationally tractable regret minimization algorithm for LTV with unknown system dynamic? }

In this paper, we study the regret minimization problem on LTV with unknown system dynamics and answer the above question affirmatively. We propose Restarting based OFU (R-OFU) and Sliding Window based OFU (SW-OFU) algorithms to find a class of linear feedback policies for minimizing the long-term cumulative dynamic regret~\cite{lin2021perturbation} across episodes. We note that our objective is thus different from adaptive regret, which focuses on (worst case) regret over intervals \cite{gradu2020adaptive,minasyan2021online}, and remains different from planning regret \cite{agarwal2021regret}. 
Both of our algorithms are based on the optimism in the face of uncertainty (OFU) principle \cite{abbasi2011improved,chowdhury2021adaptive}. This encourages our algorithms to explore for optimal solutions given the current estimation of the system dynamics. We further verify that R-OFU and SW-OFU are computationally tractable. This is because only a mini-batch of historical data in the current epoch (or sliding window) is utilized for online planning. With proper configuration, our algorithm attains sublinear regret $O(T^{2/3})$. 

We further demonstrate the versatility and practicality of our algorithm with extensive experiments on switching and time-variant systems. Our empirical results corroborate our theoretical findings with respect to the regret and cost.

\textbf{Paper Structure.} We first present the necessary definitions and problem formulation in Section 2. Then, we present the detailed algorithms in Section 3. In Section 4, we provide the main results with respect to the R-OFU and SW-OFU algorithms. The analysis as well as proof sketches of the proposed algorithms are presented in Section 5. Lastly, in Section 6, the results and details of the experiments are provided. Proofs and other details are deferred to the Appendix.

\begin{table*}[!t]
 \small
 \center
 \caption{Representative (online)  control  algorithms for regret minimization.  }
 \setlength{\tabcolsep}{1.8mm}
\begin{tabular}{|c|c|c|c|c|c|}
\hline
  \textbf{Ref.}  &  \textbf{Dynamic} & \textbf{Environment} &  \textbf{Type of Regret}  &   \textbf{Knowledge }  &  \textbf{Regret Bound\ (in terms of $T$)}   \\ \hline
  \cite{cassel2020logarithmic}  &  Linear   & Time-Inv. & Cumulative Regret & Partial & $O(\log^2T)$ \\
\cite{simchowitz2020naive}  &  Linear   & Time-Inv. & Cumulative Regret & No & $O(\sqrt{T})$\\
\cite{chen2021black}  &  Linear   & Time-Inv. & Cumulative Regret & Yes  & $\widetilde{O}(T^{2/3})$	\\
\cite{Kakade2020information}  &  Nonlinear   & Time-Inv. & Cumulative Regret & Partial &  $O(\sqrt{T})$\\
\cite{ho2021online}  &  Nonlinear   & Time-Inv. & Mistake & Partial  &- - - - - -  \\ \hline
\cite{lin2021perturbation}  &  Linear   & Time-Var. & Dynamic Regret & Yes &  $O({\lambda^kT})$ \\
\cite{gradu2020adaptive}  &  Linear   & Time-Var. & Adaptive Regret & Yes  &- - - - - - \\
\cite{agarwal2021regret}  &  Nonlinear  & Time-Var. & Planning Regret & Partial & - - - - - -\\\hline
\cite{minasyan2021online}  &  Linear   & Time-Var. & Adaptive Regret & No & $O({{e}^{\Omega \left( {{d}_{x}}{{d}_{u}} \right)}}T^{1-\frac{1}{2(d_xd_u+3)}})$ \\ 
Ours  &  Linear   & Time-Var. & Dynamic Regret & No & $\widetilde{O}(T)$ (epoch $<$ episode length)  \\
Ours  &  Linear   & Time-Var. & Dynamic Regret & No & $\widetilde{O}(T^{2/3})$ (epoch $\geq$ episode length)\\
 \hline
\end{tabular}
\label{table:methods}
\end{table*}

\section{Problem Setting}

\paragraph{Notation}
We use ${{\left\| A \right\|}_{F}}=\sqrt{{{\left\langle A,A \right\rangle }_{F}}}=\sqrt{\text{Trace}{{\left\langle A*A \right\rangle }}}$ to denote the Frobenius norm of matrix $A$. For two matrices $X$ and $Y$, we also define $\left\| X \right\|_{Y}^{2}=\text{Trace}({{X}^{\top}}YX)$. $\mathbb{E}[X]$ denotes the expectation of a random variable $X$ and $x\vee y$ denotes the maximum between $x,y\in \mathbb{R}$.

\subsection{Problem Formulation}

We consider the episodic time-varying linear quadratic regulator (LQR) setting with $K$ episodes and $H$ steps. We let $x\in \mathcal{X}\in {{\mathbb{R}}^{\text{n}}}$ denotes the vector of system state, $u\in \mathcal{U}\in {{\mathbb{R}}^{\text{m}}}$ denotes the vector of control input, and $w_t$ denotes the system noise, which is zero-mean. In each episode $k$, the agent starts from a random initial state sampled from the initial distribution $x_{k,h=1} \sim \rho$, and executes $H-1$ control steps to finish the episode. Then the agent starts over from $h=1$ for the $(k+1)$-th episode with a new initial state $x_{k+1,h=1}$ sampled from $\rho$. This process repeats for the specified number of episodes $K$. The dynamic of the $k$-th episode on the time-varying LQR system is described as
\begin{equation}
x_{k,h+1}=A_{k,h}x_{k,h}+B_{k,h} u_{k,h}+w_{k,h} \,,
\label{eq1}
\end{equation}

with a quadratic cost function $c_{k,h}=x_{k,h}^{\top}Q x_{k,h}+u_{k,h}^{\top}R u_{k,h} $.
This dynamic is governed by unknown time-varying matrices $A_{k,h}$ and $B_{k,h}$, while $Q$, $R$ are known positive definite matrices. The key difference between the non-stationary equation \eqref{eq1} and existing stationary LQR learning systems (i.e. ${{\widehat{x}}_{k,h+1}}=A{{\widehat{x}}_{k,h}}+B{{u}_{k,h}}+{{w}_{k,h}}$) \cite{abbasi2011regret,wang2020episodic} is that the transition matrix $A_{k,h}$ and $B_{k,h}$ evolve with the time step $h$ and the episode $k$. Remark that the dynamics vary between different episodes, which makes information non-transferable between them.

The goal is to design a control policy $\pi:[H]\times \mathcal{X} \to \mathcal{U}$ that minimizes the accumulated expected cost within each episode $k\in [K]$
 \begin{equation}
J_{k,h}^{\pi}(x)=\mathbb{E}_{\pi}\left[ \textstyle{\sum}_{h'={{h}_{}}}^{H}{{{c}_{k,h'}}}|{{x}_{{{k,h}}}}={{x}} \right] \,,
\label{eq3}
\end{equation}
where $\mathbb{E}_{\pi}$ denotes expectation over the random trajectories generated by $\pi$ starting from $x$ at $(k,h)$. 

Let $\Theta_{k,h}^*=[A_{k,h},B_{k,h}]^\top \in {{\mathbb{R}}^{(n+m)\times n}}$~\footnote{ We make $\Theta^*$ and $\Theta_*$ equivalent and interchangeable through the whole paper.}. The optimal policy $\pi^*$ can then be expressed as  
\begin{equation}
\pi _{k,h}^{*}=K_{k,h}(\Theta _{k,h}^{*}){{x}_{k,h}}\,,
\label{eq4}
\end{equation}
\noindent
where $K_{k,h}(\Theta _{k,h}^{*})$ is the gain of the control policy
 \begin{equation}
{{K}_{k,h}}(\Theta _{k,h}^{*})=-{{(R+B_{k,h}^{\top}{{P}_{k,h}}(\Theta _{k,h}^{*}){{B}_{k,h}})}^{-1}}B_{k,h}^{\top}{{P}_{k,h}}(\Theta _{k,h}^{*}){{A}_{k,h}} \,,
\label{eq5}
\end{equation}
and ${{P}_{k,h}}(\Theta _{k,h}^{*})$ is the solution to the Riccati equation~\cite{bertsekas2019reinforcement}.

The optimal cost is thus given by
\begin{align*}  
J_{k,h}^{*}(\Theta _{k,h}^{*},x)={{x}^{\text{T}}}{{P}_{k,h}}(\Theta _{k,h}^{*})x+\textstyle{\sum}_{h'=h}^{H}{\mathbb{E}\left[ w_{k,h'}^{T}{{P}_{k,h'+1}}(\Theta _{k,h}^{*}){{w}_{k,h'}} \right]} \,.
  \label{optcost}
\end{align*}

Intuitively, controlling such system is intractable, the natural choice is playing zero control input $u_{k,h}=0$. However, we assume that the system dynamic evolves slowly according to the following Assumption~\ref{A3}. Therefore, at each time step, an optimal controller is optimistically computed based on the current estimation, similar to a LTI system. 

The agent's performance over $K$ episodes is measured by the cumulative (pseudo) dynamic regret $\mathcal{R}$ with respect to the true system dynamics of the model, which is also time dependent. Formally, this is referred to as the dynamic regret. 
\begin{definition}[Dynamic regret] Over $K$ episodes, the dynamic regret of an agent is 
 \begin{equation}
\mathcal{R}(K)=\textstyle{\sum}_{k=1}^{K}{J_{1}^{{{\pi }_{k}}}({{\Theta }_{k}^*},{{x}_{k,1}})}-J_{1}^{*}({{\Theta }_{k}^*},{{x}_{k,1}}) \,,
\label{icmleq4t}
\end{equation}

where${J_{1}^{{{\pi }_{k}}}({{\Theta }_{k}^*},{{x}_{k,1}})}$ is the expected cost under chosen $\pi_k$ at the episode $k$, $\Theta _{k}^{*}=[\Theta_{k,1}^{*},\,\,\Theta_{k,2}^{*},\,\,...\,\,,\,\Theta _{k,H}^{*} ]$, and $J_{1}^{*}({{\Theta }_{k}^*},{{x}_{k,1}})$ is the expected cost under optimal control policy for the episode $k$. 
\end{definition}
We make the following assumptions on controllability and boundness to make this problem tractable. Note that similar assumptions are also used in the literature \cite{abbasi2011regret,wang2020episodic,abeille2020efficient}.

\begin{assumption}
The true system $\Theta^\ast$ is controllable and open-loop stable (i.e., $\text{Rank}\left( \left[ {{B}_{k,h}}\,\,\,{{A}_{k,h}}{{B}_{k,h}}\,\,\,A_{k,h}^{2}{{B}_{k,h}}\,\,...\,\,A_{k,h}^{n-1}{{B}_{k,h}} \right] \right)=n)$ and bounded ${{\left\| {{\Theta }^{*}} \right\|}_{F}}\le 1$. There also exits constants $\upsilon, \upsilon_A, \upsilon_B$, and $\upsilon_w$ such that $\left\| {{A}_{k,h}} \right\|\le {{\upsilon }_{A}}<1$, $\left\| {{B}_{k,h}} \right\|\le {{\upsilon }_{B}}<1$, $\left\| {{w}_{k,h}} \right\|_2\le {{\upsilon }_{w}}<1$, and $\left\| {R} \right\|, \left\| {Q} \right\|\le {{\upsilon }}$. For $k\ge 1$, the states $\left\| {x}_{k,1} \right\|\le 1$. Further, ${\upsilon }_{w}+\Upsilon{\upsilon }_{B}+{\upsilon }_{A}\le 1$ with $\Upsilon $ being a constant. 
\label{A1} 
\end{assumption}

\begin{assumption}
We assume that the total system variability on every episode is bounded,
\begin{equation*}
    \textstyle{\sum}_{h=1}^{H-1}{{{\left\| {{\Theta }_{k,h+1}}-{{\Theta }_{k,h}} \right\|}_{F}}}\le {\mathcal{B}_{H}} \,, \quad \forall k \in K.
\end{equation*}
\label{A3}
\end{assumption}

\begin{assumption}
Let $\left\{ \mathcal{F}_{k,h} \right\}^{\infty}_{h=0}$ be a filtration generated by the random variables $\left\{ x_{k,h}, u_{k,h} \right\}^{\infty}_{h=1}$. We assume that $\left\{ w_{k,h} \right\}_{h\ge 1}$ is a vector valued martingale process adapted to filtration $\left\{ \mathcal{F}_{k,h} \right\}_{h\ge0}$. Further, let $\eta_t$ be a sub-Gaussian random vector with a fixed constant $R>0$, and for any $\chi \in \mathbb{R}^n$,
\begin{equation*}
    \mathbb{E}\left[ \exp \left( {{\chi }^{\top}}{{w }_{k,h}} \right)|{\mathcal{F}_{k,h-1}} \right]\le \exp \left( \frac{{{R}^{2}}{{\left\| \chi  \right\|}^{2}}}{2} \right) \,, \quad \forall h\ge 1.
\end{equation*}
\label{A4}
\end{assumption}

\vspace{-.5cm}
\section{Algorithms}

In this section, we propose R-OFU and SW-OFU algorithms to minimize dynamic regret $\mathcal{R}$ under LTV systems. Both algorithms conduct planning and policy execution in a fully online fashion. In the online planning step, the algorithm estimates the current $\Theta_h^*$ based on historical data from the current phase with restarting (R) or sliding Window (SW) strategies. In the policy execution step, we apply greedy policy search with optimism in the face of uncertainty (OFU). Specifically, a better model estimation (in terms of cost) is searched under a confidence region and the model with the best estimated dynamics is selected for solving the Riccati equation.

\subsection{Online Planning}
\label{S3.1}
The key ingredients of our online planning phase are the restarting and sliding window strategy, which allows us to only use the data from the current epoch for estimating ${{\Theta}^*_{h}}$. This thus greatly reduces the computation overhead and allows for a tractable algorithm. 

To shorthand the notation, we write the system parameters ${{z}_{k,h}}={{[ x_{k,h}^\top,\,\,u_{k,h}^\top ]}^\top}$, ${{Z}_{k,h}}=\left[ z_{k,h}^\top \right]$, $X_{k,h}^{next}=\left[ x_{k,h+1}^\top \right]$, and $W_{k,h}^{{}}=\left[ w_{k,h}^\top \right]$ for step $h\in [H]$ in the episode $k\in [K]$. Also, in the following paper we abbreviate the nomenclature when referring to any episode $k$ as $x_{h}=x_{k,h}$, similarly we define $z_h, X_h, X_h^{\text{next}}, Z_h$ and $W_h$.

\textbf{Restarting (R):} Within each episode, the restarting least-square ridge regression estimator is implemented using the historical data in the current epoch,
 \begin{equation}
{{\Theta }_{h}}=\arg\textstyle{\min}_\Theta   ||\Theta ||_{{\lambda I}}^{2}+\textstyle{\sum}_{s={{h}_{0}}}^{h-1}{\,||X_{s}^{\text{next}}-{{Z}_{s}}\Theta ||_{F}^{2}}\,,
\label{eq51t}
\end{equation}
where $h_0$ is the starting point of the current epoch. Then, $\Theta_h$ admits a closed-form solution 
$$\Theta_h=\mathcal{V}_h^{-1}\mathcal{U}_h={{( \textstyle{\sum}_{s={{h}_{0}}}^{h-1}{Z_{s}^\top{{Z}_{s}}+\lambda I} )}^{-1}} ( \textstyle{\sum}_{s={{h}_{0}}}^{h-1}{Z_{s}^\top X_{s}^{\text{next}}}  ) \,. $$

\textbf{Sliding Window (SW):} Consider a sliding window of length $\mathcal{W}$, $\left( 1\vee \left( h-\mathcal{W} \right) \right):\left( h-1 \right)$, with observation history $\left\{ \left( {{Z}_{s}},\,X_{s}^{\text{next}} \right) \right\}_{s=1\vee (h-\mathcal{W})}^{h-1}$, the sliding window least-square ridge regression estimator is defined as
 \begin{equation}
{{\Theta }_{h}}=  \textstyle{\arg\min}_\Theta\,\,||\Theta ||_{{
\lambda I}}^{2}+\textstyle{\sum}_{s=1\vee (t-\mathcal{W})}^{h-1}{\,||X_{s}^{\text{next}}-{{Z}_{s}}\Theta ||_{F}^{2}} \,.
\label{eq51t2}
\end{equation}

Similar to the closed form solution of (\ref{eq51t2}), the solution of the SW estimator is $\Tilde{\mathcal{V}}_h^{-1}\Tilde{\mathcal{U}}_h$ where
\begin{equation*}
\begin{aligned}
    \Tilde{\mathcal{V}}_h = \textstyle{\sum}_{s=1\vee (h-\mathcal{W})}^{h-1}{Z_{s}^\top {{Z}_{s}}+\lambda I}\,, \quad
    \Tilde{\mathcal{U}}_h = \textstyle{\sum}_{s=1\vee (h-\mathcal{W})}^{h-1}{Z_{s}^\top X_{s}^{\text{next}}}
\end{aligned}
\end{equation*}

\textbf{Comparing the restarting and sliding window strategy.} The restarting and sliding window strategies are two common strategies used in non-stationary online estimation literature \cite{jiang2004novel,chen2021combinatorial,minasyan2021online}. Both strategies are depicted in Figure \ref{ICMLF1}. Specifically, the restarting R strategy within each epoch, it discards data and re-identifies the model.

\begin{wrapfigure}[11]{r}{0.6\textwidth}
    \vspace{-0.5cm}
  \begin{center}
    \includegraphics[width=0.6\textwidth]{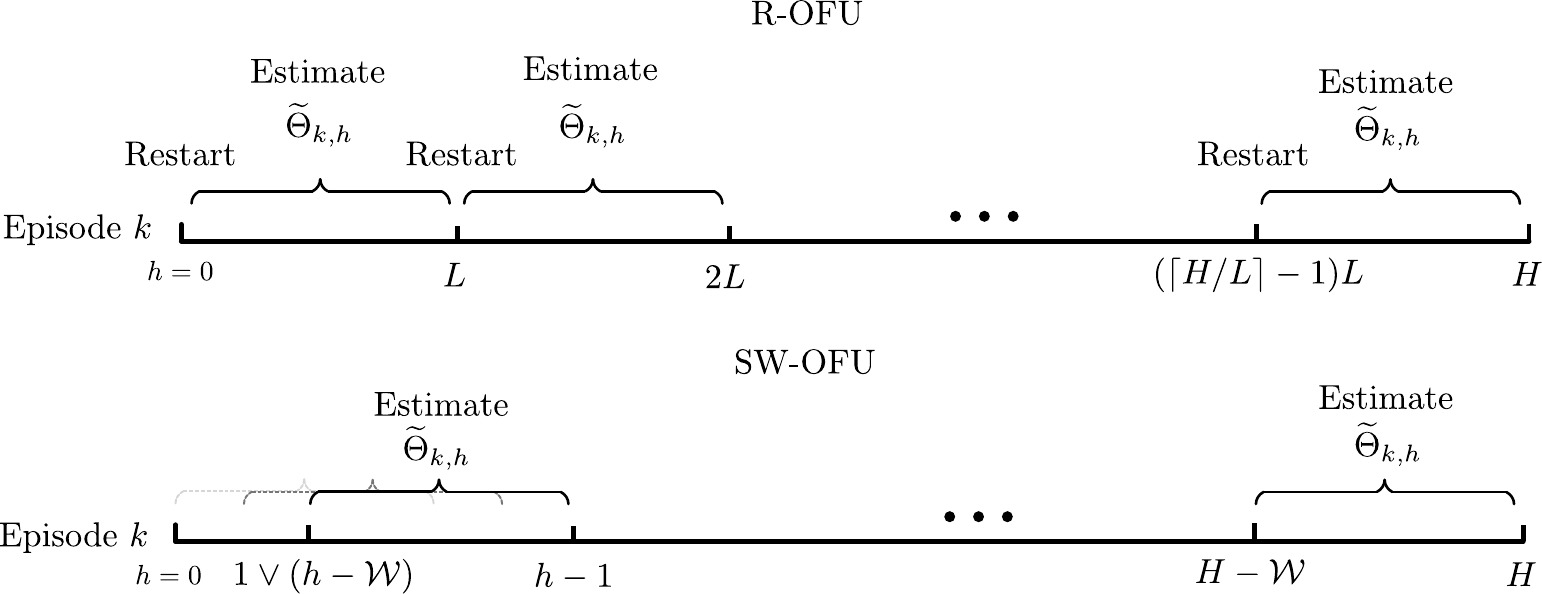}
  \end{center}
  \caption{Comparison between R-OFU and SW-OFU.}
  \label{ICMLF1}
\end{wrapfigure}

In contrast, SW draws and throws out data continuously using a sliding window. Therefore, the R adapts better in abruptly changing systems, especially with a given detecting mechanism \cite{chen2021combinatorial}. Though the sliding window strategy achieves better performance in slowly changing dynamics. This phenomenon will be further discussed throughout the experiments in Section~\ref{S6}.

\subsection{Policy Execution}
\label{policy_execution}
We integrate into our algorithms the OFU principle during the policy execution step. Therefore, after estimating the dynamics of the system, we optimistically select the best model within a confidence interval around the initial estimation. This allows our algorithm to explore in the uncertain environment~\cite{sutton2018reinforcement}. We then greedily select our action with respect to our chosen model using~\eqref{eq4}. The overall description of the methods is summarized in Algorithm~\ref{ROFU} and Algorithm~\ref{SWOFU}.

{\bf High Probability Confidence Set.}
Based on the estimation of $\Theta_h$ obtained in the planning step (section \ref{S3.1}), we construct a high probability confidence set for the system model $\Theta^*$. Inspired by the analysis of \cite{wang2020episodic}, we design the confidence set as follows.

\begin{lemma} For any $h\in [H]$ of an episode and $\delta \in (0,1)$, with probability at least $1-\delta$, the estimation error is upper bounded as
\begin{equation}
\begin{aligned}
  {{\left\| \Theta _{h}^{*}-{{\Theta }_{h}} \right\|}_{\mathcal{V}_h}} \le {{\zeta }_{h}}(\delta ) \,, \quad
  {{\left\| \Theta _{h}^{*}-{{\Theta }_{h}} \right\|}_{\Tilde{\mathcal{V}}_h}} \le {\Tilde{\zeta }_{h}}(\delta ) \,.
 \label{finalconfidence}
\end{aligned}
\end{equation}
where 
\begin{equation}
\begin{aligned}
 {{\zeta }_{h}}(\delta )&=\sqrt{\lambda}+\upsilon_w \sqrt{2\ln( \textstyle \frac{2}{\delta}) +n\ln \frac{\det ({\mathcal{V}_{h}})}{\det (\lambda I)}} + \textstyle \frac{\sqrt{L(m+n)}}{\sqrt{\lambda }} {\mathcal{B}_{H}} \,, \\ 
   {\Tilde{\zeta }_{h}}(\delta )&=\sqrt{\lambda}+\upsilon_w \sqrt{2\ln (\textstyle \frac{2H}{\delta})+n\ln \textstyle \frac{\det ({\mathcal{V}_{h}})}{\det (\lambda I)}}+\textstyle \frac{\sqrt{\mathcal{W}(m+n)}}{\sqrt{\lambda }} {\mathcal{B}_{H}}\,.
 \label{icml13}
\end{aligned}
\end{equation}
are the radius of confidence region for R with length $L$, SW with length $\mathcal{W}$ respectively and $\upsilon_w$ is from Assumption \ref{A1}.
\label{Hoe}
\end{lemma}

With Lemma \ref{Hoe} and the estimator ${{\Theta }_{h}}$ from the online planning step, our algorithm maintains confidence radius,
\begin{equation}
\begin{aligned}
\mathcal{C}_h (\delta )=\left\{ \Theta :\,{{\left\| \Theta -{{\Theta }_{h}} \right\|}_{{{{\mathcal{V}_h}}}}}\le {{\zeta }_{h}}(\delta )\right\}\ \,, \quad
\Tilde{\mathcal{C}}_h (\delta )=\left\{ \Theta :\,{{\left\| \Theta -{{\Theta }_{h}} \right\|}_{\Tilde{\mathcal{V}}_h}}\le {\Tilde{\zeta }_{h}}(\delta )\right\} \,.
\label{icmleeq14}
\end{aligned}
\end{equation}

{\bf OFU-Based Action Search.}

Within the confidence set $\mathcal{C}_h{(\delta )}$ or $\Tilde{\mathcal{C}}_h (\delta )$, we adopt the OFU principle to compute an optimistic estimate of ${\widetilde{\Theta }}_{h}$,
\begin{equation}
 {\widetilde{\Theta }_{h}}\in \underset{\Theta \in \mathcal{C}_h{(\delta )}}{\mathop{\arg \min }}\,{{J}_1^{*}}(\Theta ,{{x}_{k,1}})
\label{icmleeq15}
\end{equation}
\noindent
where ${{J}_1^{*}}(\Theta ,{{x}_{k,1}})$ is the optimal cost when the true dynamic is $\Theta$. Then, the agent computes the control following the policy
 \begin{equation}
 {{u}_{h}}={{\pi }_{h}}({{x}_{k,h}})={{K}_{h}}({{\widetilde{\Theta }}_{h}}){{x}_{k,h}} \,,
\label{icmleeq16}
\end{equation}
\noindent
where the gain ${{K}_{h}}({{\widetilde{\Theta }}_{h}})$ can be calculated through~\eqref{eq5}.

\begin{algorithm}[t]
    \caption{R-OFU based online control algorithm}
    \label{ROFU}
    \begin{algorithmic}[1]  
        \REQUIRE Number of episodes $K$, time horizon $H$, epoch size $L$, regularization strength $\lambda$
    
        \FOR {Episode $ k = 1,2,...,K $;  }
            \STATE {Set epoch counter $j=1$}
            \WHILE {$ j \leq \lceil H/L \rceil $}
            \STATE {Set $\kappa =(j-1)L$ and initialize $\mathcal{V}_{\kappa}=\lambda I$}

           \FOR {$h=\kappa+1$,....,$\kappa+L-1$  }
            \STATE {Compute ${{\Theta }_{h}}=\mathcal{V}_h^{-1}\mathcal{U}_h$ with ${{\zeta }_{h}}(\delta )$ computed from (\ref{icml13})}
            
            
            \STATE {Construct high confidence set $\mathcal{C}_h{(\delta )}$ and select $ {{\Theta }_{h}}\in \underset{\Theta \in \mathcal{C}_h{(\delta )} }{\mathop{\arg \min }}\,{{J}_1^{*}}(\Theta ,{{x}_{k,1}})$}
            
            
            \STATE {Implement control $ {{{u}}_{k,h}}={{K}_{h}}({{{\Theta }}_{h}}){{x}_{k,h}}$ and observe cost $c_{k,h}$, $Z_{k,h}$ and $X_{k,h}^{next}$}
            

        \ENDFOR
        
        \STATE {Set $j=j+1$ }
           
        \ENDWHILE
        \ENDFOR
    \end{algorithmic} 
\end{algorithm}

\begin{algorithm}[t]
    \caption{SW-OFU based online control algorithm}
    \label{SWOFU}
    \begin{algorithmic}[1]  
        \REQUIRE Number of episodes $K$, time horizon $H$, sliding window size $\mathcal{W}$, regularization strength $\lambda$
    
        \FOR {Episodes $ k = 1,2,...,K $;  }
            \STATE {Initialize $\Tilde{\mathcal{V}}_{k,0}=\lambda I$}

           \FOR {$h=1$,....,$H$  }
            \STATE {Compute ${{\Theta }_{h}}=\Tilde{\mathcal{V}}_h^{-1}\Tilde{\mathcal{U}}_h$ with set ${\Tilde{\zeta }_{h}}$ computed from (\ref{icml13})}
            
            
            \STATE {Construct high confidence set $\mathcal{C}_h{(\delta )}$ and select $ {{\Theta }_{h}}\in \underset{\Theta \in \mathcal{C}_h{(\delta )}}{\mathop{\arg \min }}\,{{J}_1^{*}}(\Theta ,{{x}_{k,1}})$}

            
            \STATE {Implement control $ {{u}_{k,h}}={{K}_{h}}({{{\Theta }}_{h}}){{x}_{k,h}}$ and observe cost $c_{k,h}$, $Z_{k,h}$ and $X_{k,h}^{next}$}
            

        \ENDFOR
        
           
        \ENDFOR
    \end{algorithmic} 
\end{algorithm}

To ensure that equation~\eqref{icmleeq16} is well-defined and satisfies the stability condition, we establish the following propositions. The detailed proofs and discussion is deferred to the Appendix~\ref{ICMLT4.5}.

\begin{proposition}
The region encompassed by the high probability confidence set (\ref{icmleeq14}) is closed and bounded.

\label{ICMLSTA}
\end{proposition}

\begin{proposition}
Given any ${\widetilde{\Theta }_{h}}$ in (\ref{icmleeq15}), the gain of the controller ${{K}_{h}}({{\widetilde{\Theta }}_{h}})$ is well defined.

\label{ICMLSTA2}
\end{proposition}

\section{Main Results and Analysis}
\label{section4}
With the R-OFU and SW-OFU algorithms, we obtain the following dynamic guarantees for the unknown time-varying LTV system.

\begin{theorem}[Dynamic regret with R-OFU]
Algorithm (\ref{ROFU}) achieves a high probability dynamic regret bound
\begin{align*}
    \mathcal{R}(K)
    = \ & O\left( {{H}^{\frac{3}{2}}}\sqrt{K} \right) +O\left( HK\vartheta (\delta )\sqrt{(n+m)\ln \left( 1+\tfrac{HK}{(n+m)\lambda } \right)} \right) \\ 
    & \ +O\left( HK\left( \ln \tfrac{1}{\delta }+n\left( n+m \right)\ln \left( 1+\tfrac{HK}{(n+m)\lambda } \right) \right) \right) \,,
\end{align*}
where $\delta\in (0,1)$ is the probability parameter, $\vartheta (\delta )=\sqrt{\lambda }+\sqrt{\frac{L(m+n)}{\lambda }}{\mathcal{B}_{H}}$, and $L$ is the epoch size.
\label{ICMLT1}
\end{theorem}
\begin{theorem}[Dynamic regret with SW-OFU]
Algorithm (\ref{SWOFU}) achieves a high probability dynamic regret bound 
\begin{align*}
    \mathcal{R}(K)
    = \ & O\left( {{H}^{\tfrac{3}{2}}}\sqrt{K} \right) +O\left( HK\Tilde{\vartheta} (\delta )
  \sqrt{(n+m)\ln \left(1+\tfrac{HK}{(n+m)\lambda } \right)} \right) \\ 
    & \ +O\left( HK\left( \ln\tfrac{H}{\delta }+n\left( n+m \right)\ln \left( 1+\tfrac{HK}{(n+m)\lambda } \right) \right) \right) \,,
\end{align*}
where $\Tilde{\vartheta} (\delta )=\sqrt{\lambda }+\sqrt{\frac{\mathcal{W}(m+n)}{\lambda }}{\mathcal{B}_{H}}$, $\delta\in (0,1)$ is the probability parameter, and $\mathcal{W}$ is the sliding window size. 
\label{ICMLT2}
\end{theorem}

\noindent
\begin{corollary}
From Theorems~\ref{ICMLT1} and~\ref{ICMLT2}, the dynamic regret is sublinear in $K$ when the sliding window size or the restarting epoch length is set to be larger than $H$.
\end{corollary}

\begin{theorem}[Dynamic regret with with a larger size of epoch]
Our R-OFU algorithm achieves a high probability dynamic regret bound with a larger size of $L\ge H$
\begin{align*}
    \mathcal{R}(K)
    = \ & \widetilde{O}\left(L\mathcal{B}_{HK}+HK\sqrt{\frac{1}{L}} +\sqrt{HK}\sqrt{\mathcal{B}_{HK}}L^{1/4} \right) \,,
\end{align*}
where $\mathcal{B}_{HK}$ is the total variation budget along the whole steps. By setting $L=L^*=(HK)^{2/3}{\mathcal{B}_{HK}^{-2/3}}$, we achieve a minmax near-optimal dynamic regret $\widetilde{O}\left( {{\left( HK \right)}^{2/3}}\mathcal{B}_{HK}^{1/3} \right)$.
\label{ICMLT2222}
\end{theorem}

\noindent
\begin{remark}
From Theorems \ref{ICMLT1} and \ref{ICMLT2}, it is noted that R-OFU and SW-OFU achieve the same order of regret in terms of $T = HK$, while R-OFU is slightly better with a factor of $\ln H$. 
The additional $\ln H$ factor comes from the information loss due to the SW.
\end{remark}

\begin{remark}
When compared with prior results, our regret bound is much more practical. In previous work~\cite{minasyan2021online}, the regret is $\Omega\left(\exp \left( nm \right)T^{1-\frac{1}{2(nm+3)}}\right)$, which scales exponentially with the dimension of state and action space. In practice, achieving the regret bound can be computationally intractable. In contrast, the Theorems \ref{ICMLT1} and \ref{ICMLT2} achieves regret independent of the state and action space size. 
Considering that attaining polynomial regret (e.g., $T^{1-\alpha}$, $\alpha>0$) may not be even possible for unknown LTV, our results achieve a reasonable order of $T^{1.5}$ when $L < H$. Additionally, this regret bound  can be further improved to $T^{2/3}$ under proper configurations.

\end{remark}

\section{Analysis}
In this section, we present the analysis of Lemma \ref{Hoe} and Theorems \ref{ICMLT1} and \ref{ICMLT2}. 

\subsection{Analysis of High Confidence Set}
\label{ICMLHCS}
As shown in the closed-form expression of R and SW regressors, the key difference in the solutions is the term $h_0=\text{max} (1,h-\mathcal{W})$ in SW. We present details on the construction of a high confidence set for the restarting strategy, since it can also be applied to sliding window case with minor changes.

\begin{proposition}
\label{prop1}
From the closed-form solution of (\ref{eq51t}), the estimate error can be decomposed as, 

\begin{equation}
\begin{aligned}
  {{\left\| \Theta _{h}^{*}-{{\Theta }_{h}} \right\|}_{{\mathcal{V}_{h}}}} 
  {\le} \ & \underbrace{{{\left\| {{\left( \lambda I \right)}^{\frac{1}{2}}} \right\|}_{2}}}_{{\ell_1}}+\underbrace{{{\left\| \textstyle{\sum}_{s={{h}_{0}}}^{h-1}{Z_{s}^{\top}{{W}_{s}}} \right\|}_{\mathcal{V}_{h}^{-1}}}}_{\ell_2} 
 + \underbrace{{{\left\|\textstyle{\sum}_{s={{h}_{0}}}^{h-1}{Z_{s}^{\top}{{Z}_{s}}(\Theta _{s}^{*}-\Theta _{h}^{*})} \right\|}_{\mathcal{V}_{h}^{-1}}}  }_{{\ell_3}} \,.
 \label{ICMLeq18}
\end{aligned}
\end{equation}

\end{proposition}

The detailed result is described in Appendix~\ref{appendixB1}.
\begin{remark}
The term $\ell_1$ and $\ell_2$ are the estimate errors caused by the regularizer and random noise; while the last term $\ell_3$ is due to the time-varying property. Both R and SW have these three sources of estimate errors. The first and third terms are from the same bound for both R and SW. However, the bound for the second term is different among them.
\end{remark}

\noindent
The terms $\ell_1$, $\ell_2$ and $\ell_3$ can be bounded separately, as we summarized in the following lemmas~\ref{ICMLthree} and \ref{ICMLthree2}. The proof of the lemmas can be found in Appendix \ref{ICMLT4.6}.
\begin{lemma} For any $h\in H$ in an episode and $\delta \in (0,1)$ in R, with probability at least $1-\delta$, the following holds 
\begin{align*}
\ell_1= \sqrt{\lambda}\,,
\quad \ell_2 \le \upsilon_w \sqrt{2\ln \left( \tfrac{1}{\delta } \right)+n\ln \tfrac{\det ({\mathcal{V}_{h}})}{\det (\lambda I)}}\,,
\quad \ell_3 \le \tfrac{\sqrt{L(m+n)}}{\sqrt{\lambda }} \mathcal{B}_H\,.
\end{align*}

\label{ICMLthree}
\end{lemma}
\begin{lemma}  For any $h\in H$ in an episode and $\delta \in (0,1)$ in SW, with probability at least $1-\delta$, the following holds 
\begin{align*}
    \ell_1= \sqrt{\lambda} \,,
    \quad \ell_2 \le \upsilon_w \sqrt{2\ln \left( \tfrac{H}{\delta } \right)+n\ln \tfrac{\det (\Tilde{\mathcal{V}_{h}})}{\det (\lambda I)}} \,,
    \quad \ell_3 \le \tfrac{\sqrt{\mathcal{W}(m+n)}}{\sqrt{\lambda }} \mathcal{B}_H \,.
\end{align*}

\label{ICMLthree2}
\end{lemma}

\subsection{Analysis of Dynamic Regret}
\label{ICMLADE}

Armed with our analysis of high confidence set, we are now able to give an upper bound of the dynamic regret with R-OFU and SW-OFU algorithms. We first start with a careful decomposition of the dynamic regret under the good event where $\mathcal{\varepsilon}_K(\delta)=\left\{ {{\Theta }_{*}}\in \mathcal{C}_h{(\delta )}, \,\,\forall h\in[H] \right\} $.

\begin{lemma} 
Let ${{{\widetilde{P}}_{k,h}}}=P_h(\widetilde{\Theta}_{k,h})$ and $\mathcal{F}_{k,h}$ denotes all randomness before the step $(k,h)$. Under a 'good' event $\mathcal{\varepsilon}_K(\delta)=\left\{ {{\Theta }_{*}}\in \mathcal{C}_h{(\delta )}, \,\,\forall h\in[H] \right\} $, the dynamic regret $\mathcal{R}(K)$ in (\ref{icmleq4t}) is decomposed as 
\[\mathcal{R}(K)\le \sum\limits_{k=1}^{K}{\sum\limits_{h=1}^{H}{\varsigma}_{h,k}},\]
where
\begin{align*}
  {{\varsigma }_{k,h}}  
  =\ & \mathbb{E}[J_{h+1}^{{{\pi }_{k}}}({{\Theta }_{*}},{{x}_{k,h+1}})|{\mathcal{F}_{k,h}}]-J_{h+1}^{{{\pi }_{k}}}({{\Theta }_{*}},{{x}_{k,h+1}})
  +||{{x}_{k,h+1}}|{{|}_{{{\widetilde{P}}_{k,h+1}}}}  
  -||\widetilde{\Theta }_{k,h}^{T}{{z}_{k,h}}|{{|}_{{{\widetilde{P}}_{k,h+1}}}}\\
  & -\mathbb{E}\left[ ||{{x}_{k,h+1}}|{{|}_{{{\widetilde{P}}_{k,h+1}}}}|{\mathcal{F}_{k,h}} \right]+||\Theta _{*}^{T}{{z}_{k,h}}|{{|}_{{{\widetilde{P}}_{k,h+1}}}}  \,.
\end{align*}

\label{ICMLLL5}
\end{lemma}
The detailed proof of Lemma \ref{ICMLLL5} is presented in Appendix \ref{ICMLT4.7}.
Based on this decomposition, one can bound (\ref{icmleq4t}) separately using the following lemma for R strategy.
\begin{lemma}
Under Assumption 1 and event $\mathcal{\varepsilon}_K(\delta)$, we have the following dynamic regret bound with at least probability $1-2\delta$,
\begin{align*}
    & \sum\limits_{k=1}^{K}\sum\limits_{h=1}^{H}{\mathbb{E}[J_{h+1}^{{{\pi }_{k}}}({{\Theta }_{*}},{{x}_{k,h+1}})|{\mathcal{F}_{k,h}}]-J_{h+1}^{{{\pi }_{k}}}({{\Theta }_{*}},{{x}_{k,h+1}})} 
    \le O\left( \sqrt{K{{H}^{3}}\ln \frac{2}{\delta }} \right) \,, \\
    & \sum\limits_{k=1}^{K}\sum\limits_{h=1}^{H}||{{x}_{k,h+1}}|{{|}_{{{\widetilde{P}}_{k,h+1}}}}-\mathbb{E}\left[ ||{{x}_{k,h+1}}|{{|}_{{{\widetilde{P}}_{k,h+1}}}}|{\mathcal{F}_{k,h}} \right] 
    \le O\left( \sqrt{K{{H}^{}}\ln \frac{2}{\delta }} \right) \,, \\
    & \sum\limits_{k=1}^{K}\sum\limits_{h=1}^{H}||\Theta _{*}^{T}{{z}_{k,h}}|{{|}_{{{\widetilde{P}}_{k,h+1}}}}-||\widetilde{\Theta }_{k,h}^{T}{{z}_{k,h}}|{{|}_{{{\widetilde{P}}_{k,h+1}}}}  
    \le  O\left( HK{{\zeta }_{h}}(\delta )\sqrt{ \ln \frac{\det \left( {\mathcal{V}_{h}} \right)}{\det \left( \lambda I \right)}} \right) \,.
\end{align*}
\label{ICMLL1}
\end{lemma}

We present the proof sketch for Theorem \ref{ICMLT1} with algorithm R-OFU. In the case of the SW-OFU algorithm, the Proof of Theorem \ref{ICMLT2} is similar but with the difference in the radius term ${\Tilde{\zeta }_{h}}(\delta )$.
\begin{proof}
By the boundness results in Appendix \ref{ICMLT4.5}, we have
\[\ln \det \left( {\mathcal{V}_{h}} \right)\le (n+m)\ln \left( \lambda +\tfrac{HK{{(1+\gamma )}^{2}}}{n+m} \right)\,.\]
Therefore, ${{\zeta }_{h}}(\delta )$ can be rewritten as
\begin{equation*}
    {{\zeta }_{h}}(\delta )=\upsilon_w \sqrt{2\ln \left( \tfrac{2}{\delta } \right)+(n+nm)\ln (1+\tfrac{ HK(1+\gamma)^2}{ \lambda (n+m)}})+\sqrt{\lambda}+ \tfrac{\sqrt{L(m+n)}}{\sqrt{\lambda }} \mathcal{B}_H
\end{equation*}
Replacing ${{\zeta }_{h}}(\delta )$ into the third inequality in Lemma \ref{ICMLL1} and putting everything together yields the final result.
\end{proof}

Followed by the Proof of Theorem \ref{ICMLT1} and \ref{ICMLT2}, we present the Proof of Theorem~\ref{ICMLT2222} in the Appendix~\ref{PROOFNIPS}.

\section{Experiments}
\label{S6}

In this section, we provides empirical analysis of our algorithms under the following time-variant linear systems with an oracle LQR controller:

\paragraph{Switching system.} In the first scenario we consider a linear system whose dynamics are defined by $A_1 = \begin{bmatrix} 1 & 0.5 \\ 0 & 1  \end{bmatrix}$ and $B_1 = \begin{bmatrix} 0 \\ 1.2 \end{bmatrix}$ for the first $H/2$ timesteps in the episode. Then, the system switches to $A_2 = \begin{bmatrix} 1 & 1.5 \\ 0 & 1  \end{bmatrix}$ and $B_2 = \begin{bmatrix} 0 \\ 0.9 \end{bmatrix}$ for the last $H/2$ timesteps in the episode. 

\paragraph{Slowly changing system.} For the second experiment we consider the slowly changing system defined by $A = \begin{bmatrix} 1 & 1 \\ 0 & 1  \end{bmatrix}$ and $B_h = \begin{bmatrix} 0 \\ h/20 \end{bmatrix}$. In this case, $B_h$ constantly evolves with $h$.

\paragraph{Frequently switching system.} On the frequently switching model, the dynamics changes every $20$ steps. Specifically, the dynamics is randomly selected between a set of configurations whose controllability has being previously tested. The system configurations used are:\\
$A_1 = A_3 = \begin{bmatrix} 1 & 0.5 \\ 0 & 1  \end{bmatrix}$,  $B_1 = -B_3 = \begin{bmatrix} 0 \\ 1.2 \end{bmatrix}$,
$A_2 = A_4 = \begin{bmatrix} 1 & 1.5 \\ 0 & 1  \end{bmatrix}$,   $B_2 = -B_4 = \begin{bmatrix} 0 \\ 0.9 \end{bmatrix}$. 

We consider that all systems are perturbed under i.i.d. Gaussian noise i.e. $w_t = \mathcal{N}(0,0.1^2)$.  The performance of the algorithms is measured with quadratic cost function $c_h=\|x_h\|^2+\|u_h\|^2$.

In order to control the proposed unknown LTV systems, we use the R-OFU and SW-OFU algorithms. In the case of R-OFU we set the length of the epoch to $L=20$ whereas the size of the SW-OFU  window is $\mathcal{W}=20$. According to the OFU principle, we select the best model from a set of $m=50$ candidates generated using random noise $\mathcal{U}_{[-0.5,0.5]}$ along each of the search directions.

The results of the experiments are summarized in Fig.~\ref{figresults}. Regarding the regret~(\ref{icmleq4t}) and average cost the R-OFU performs better in scenarios in abruptly changing systems (switching and frequently switching), whereas the SW-OFU is better under slowly changing dynamics. This is due to the fact that R-OFU can adapt to changes more rapidly, as it discards the previous history at the start of a new epoch. While the SW-OFU takes advantage of the recent history to derive a control policy that performs better on slowly changing scenarios as it does not experience the  aggregated cost of restarting the estimation on every epoch (observe Fig.~\ref{figresults}b and 2e). Lastly, and as expected, we observe that the oracle LQR is not capable of adapting to the time-varying scenarios, getting non-stable.

\begin{figure*}[!t]
\vskip-4pt
\centering     
\subfigure[Switching system]{\label{fig:exp1}
    \includegraphics[width=.3\textwidth]{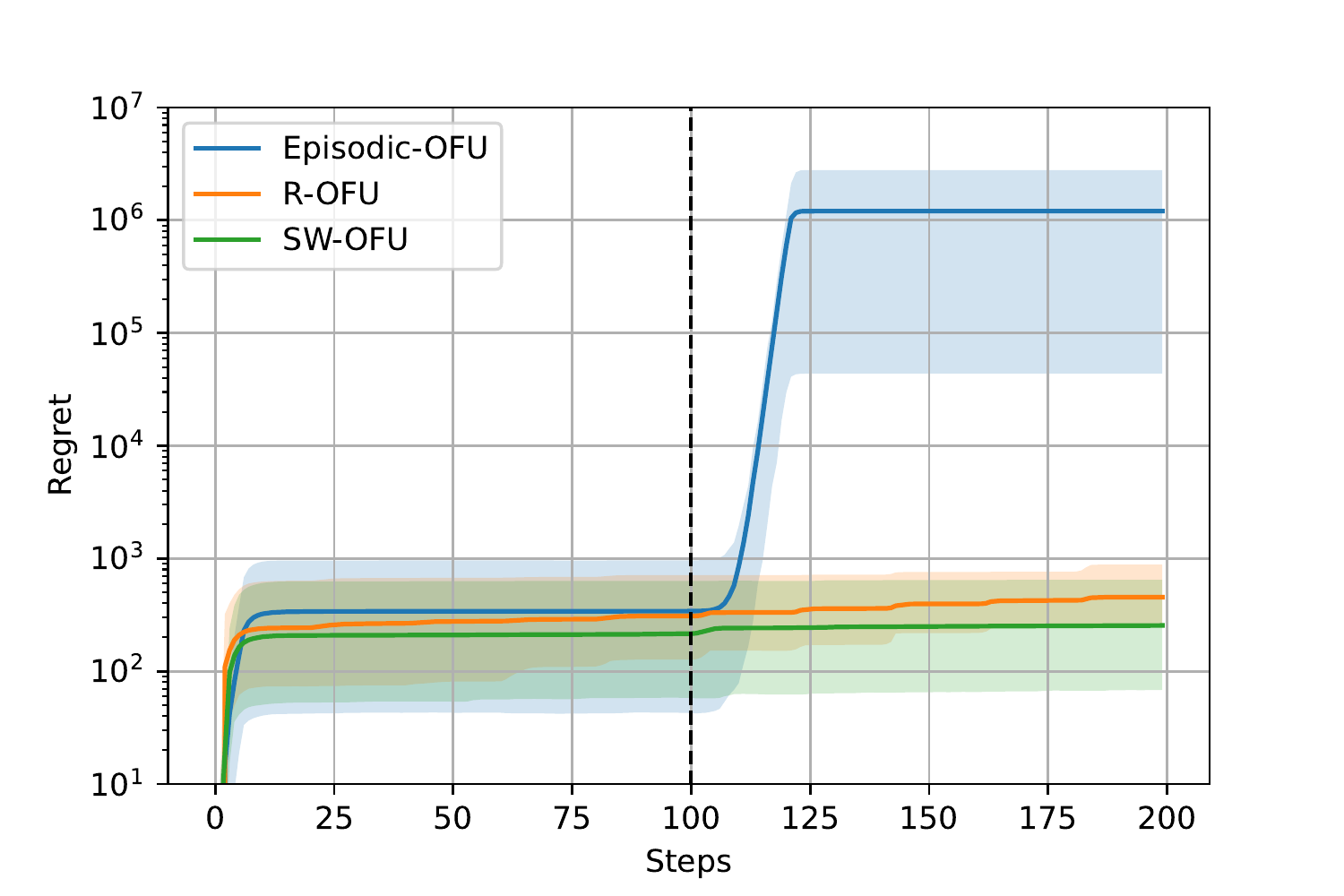}}
\subfigure[Changing system]{\label{fig:exp2}
    \includegraphics[width=.3\textwidth]{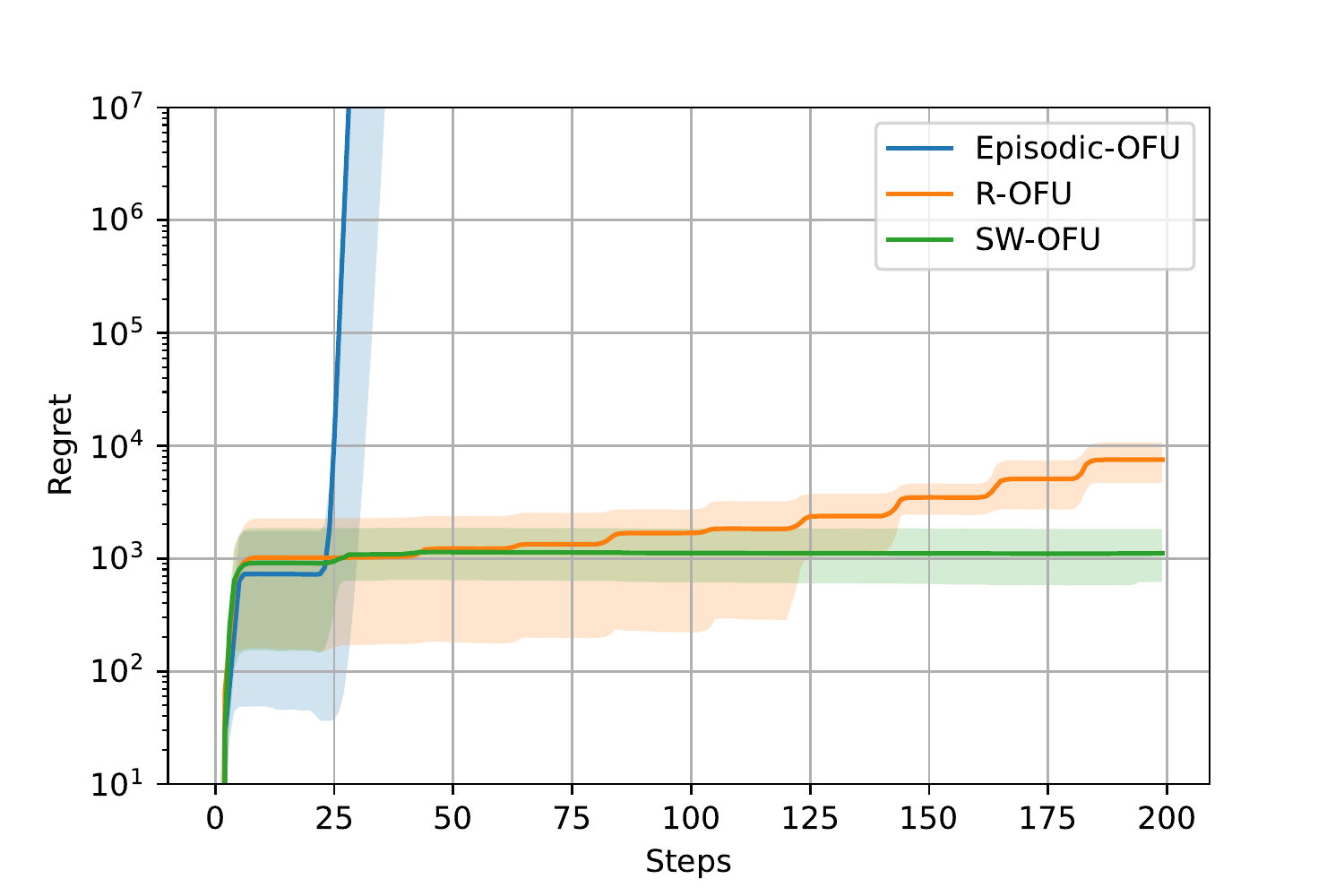}}
\subfigure[Frequently switching system]{\label{fig:exp3}
    \includegraphics[width=.3\textwidth]{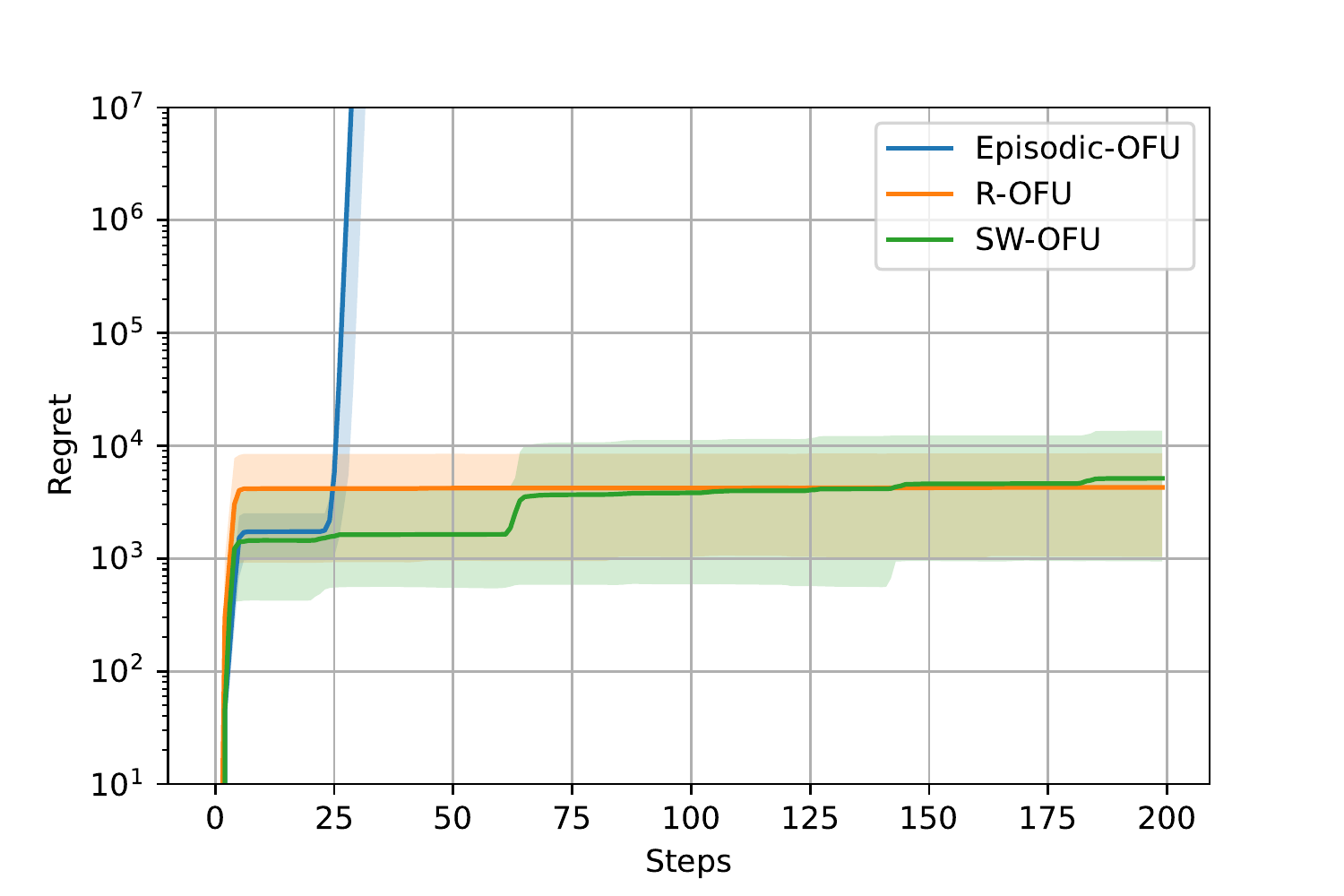}}\\ \vskip-15pt
\subfigure[Switching system]{\label{fig:exp4}
    \includegraphics[width=.3\textwidth]{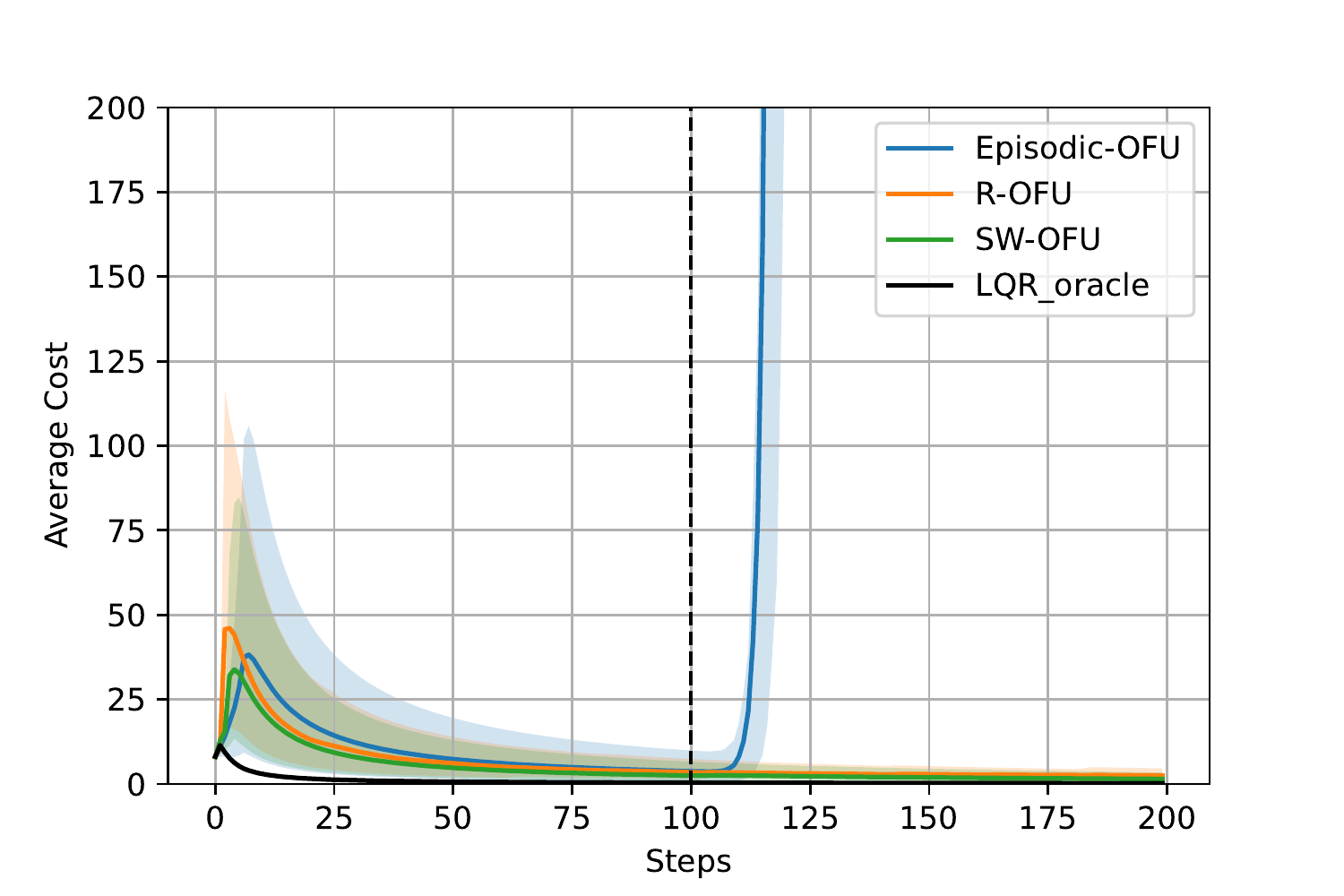}}
\subfigure[Changing system]{\label{fig:exp5}
    \includegraphics[width=.3\textwidth]{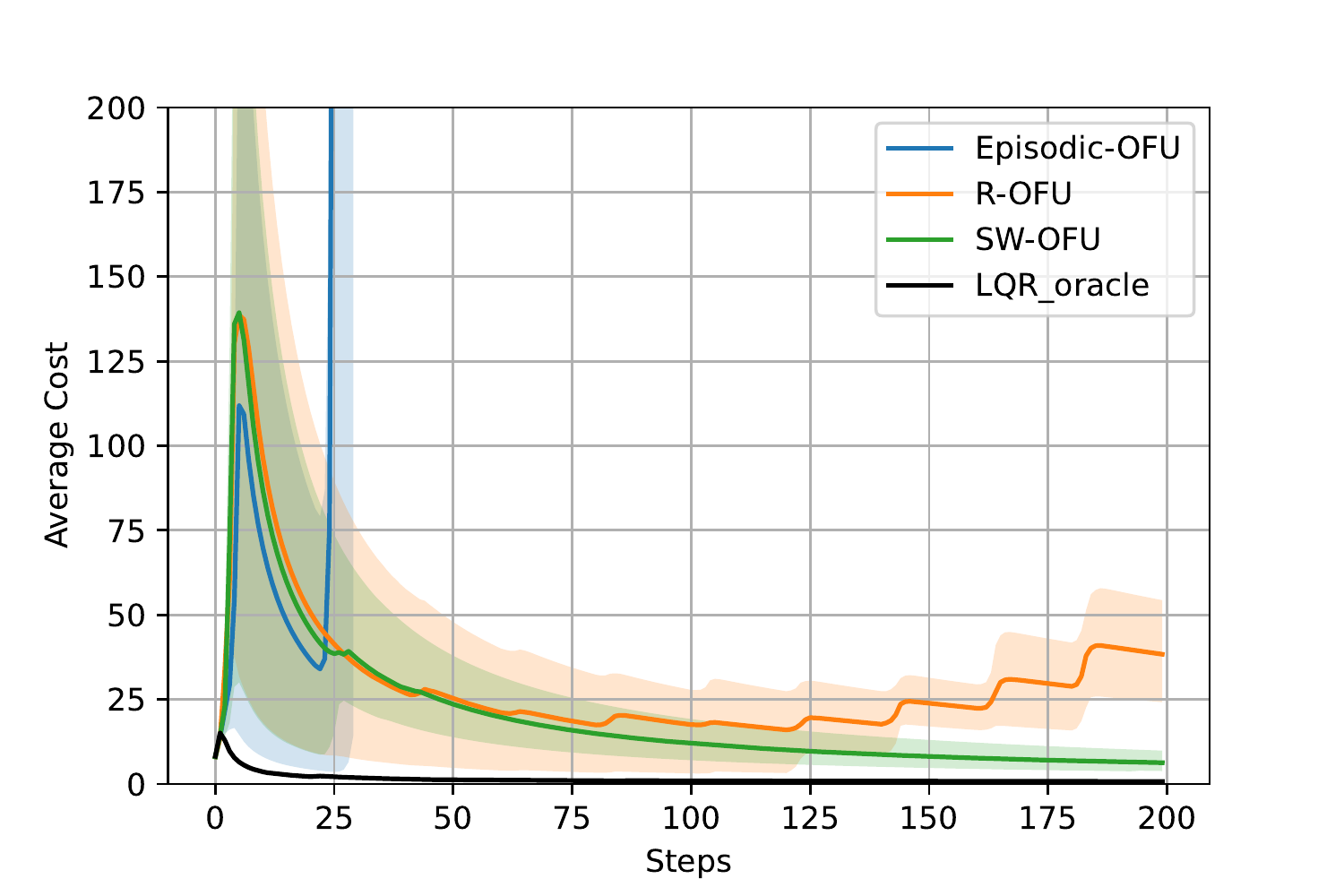}}
\subfigure[Frequently switching system]{\label{fig:exp6}
    \includegraphics[width=.3\textwidth]{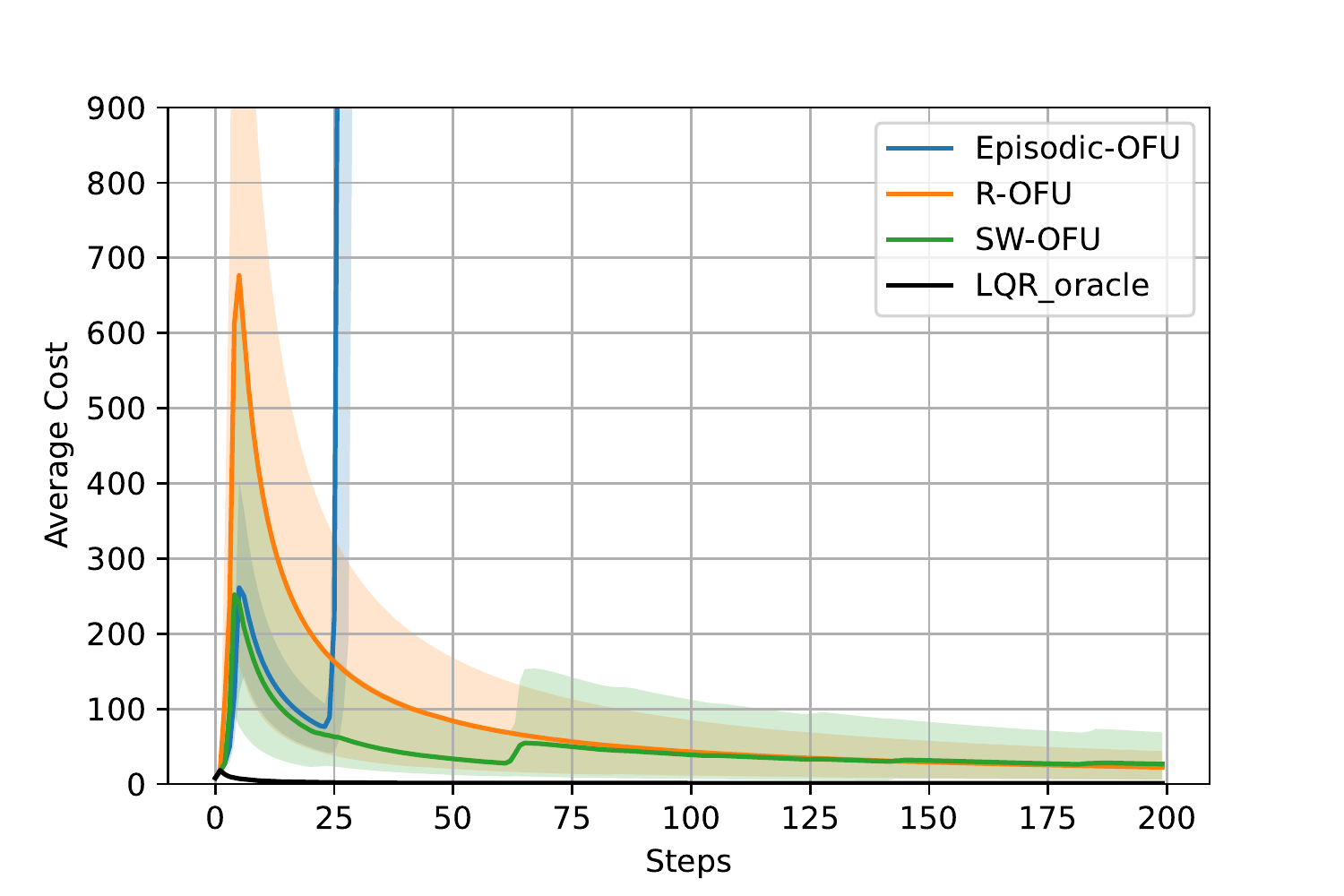}}
    \vskip-7pt
\caption{Performance comparison in the time-variant linear systems. On the top, the regret for each of the controllers (from left to right) on the switching system, slowly changing and frequently switching environments. All experiments are simulated under Gaussian perturbations, i.e. $w_t = \mathcal{N}(0,0.1^2)$. The averaged results over 5 runs are plotted and the confidence intervals are shadowed in the picture.}
\vskip-10pt
\label{figresults}
\end{figure*}

\vspace{-0.3cm}
\section{Conclusion and Future Work}

This paper studies the regret minimization problem for LTV control, where the system dynamics are unknown and change over time along with the episodes. We propose two practical algorithms based on the R and SW strategies, that notwithstanding their simplicity achieve sublinear regret under the proper configuration. To the best of our knowledge, this is the first work to obtain a tight theoretical regret bound on this setting while being computationally tractable. 
An interesting alternative direction in the future is to incorporate detection mechanisms to adaptively sense variations of the environment. This may provide an indication of how much the size of R and SW changes, potentially paving the way towards improved algorithms with tighter dynamic regret bounds. On the algorithm side, the OFU strategy has is shown to be computational tractable but inefficient. A promising direction is using nonconvex optimization~\cite{Asaf2022}, which has shown to be efficient in LTI systems.



\bibliographystyle{unsrt}
\bibliography{refs} 


\newpage

\appendix
\onecolumn
\section{Proof for Well-Definedness in Section \ref{policy_execution}}
\label{ICMLT4.5}

\begin{proposition}
The region encompassed by high the probability confidence sets (\ref{icmleeq14}) are closed and bounded.

\label{ICMLSTA}
\end{proposition}

\begin{proof}
 By definition both ${\mathcal{C}}_h$ and $\Tilde{\mathcal{C}}_h$ are closed and bounded in the region with a constant radius ${\zeta }_{h}(\delta )$ and  ${\Tilde{\zeta }_{h}}(\delta )$.
\end{proof}

\begin{proposition}
Given any ${\widetilde{\Theta }_{h}}$ in (\ref{icmleeq15}), the gain of the controller ${{K}_{h}}({{\widetilde{\Theta }}_{h}})$ is well defined.

\label{ICMLSTA2}
\end{proposition}

\begin{proof}
 Please refer to Lemma 1 in \cite{wang2020episodic}.
\end{proof}

Based on the propositions \ref{ICMLSTA} and \ref{ICMLSTA2}, we also present several boundeness results through the following corollaries.

\begin{corollary}

Under Assumption \ref{A1}, the following holds,

\[
{{\left\| {{x}_{k,h}} \right\|}_{2}}\le 1,\,\,\,\,\,{{\left\| {{u}_{k,h}} \right\|}_{2}}\le \gamma ,\,\,\,\,{{\left\| {{z}_{k,h}} \right\|}_{2}}\le 1+\gamma \,,\]

for all $k\ge$ and $h\in [H]$.

\label{ICMLBK0}
\end{corollary}

\begin{corollary}

 Under Assumption \ref{A1}, there exists a constant $D$ such that,

\[{{\left\| {{P}_{k,h}}(\Theta ) \right\|}_{2}}\le D\,,\]

for all $k\ge$ and $h\in [H]$.

\label{ICMLBK}
\end{corollary}

\begin{corollary}

The Spectrum of matrices $\mathcal{V}_h$ and $\Tilde{\mathcal{V}}_h$ is bounded, i.e., 

\[\rho(\mathcal{V}_h),  \quad  \rho(\Tilde{\mathcal{V}}_h)\le M\,.\]

\label{ICMLBKBoun}
\end{corollary}

Corollaries \ref{ICMLBK0}-\ref{ICMLBKBoun} are consistent to the boundness results in \cite{wang2020episodic}. Thus, we can adopt the same reasoning to prove it.

\section{Proof for Proposition~\ref{prop1} and Lemmas \ref{ICMLthree}, \ref{ICMLthree2} in Section \ref{ICMLHCS}}
\label{ICMLT4.6}

\subsection{Proof of Proposition~\ref{prop1} }
\label{appendixB1}

From the closed-form solution of (\ref{eq51t}), one can verify that the estimate error can be decomposed as, 
\begin{align}
\label{icmleeq17}
  & \,\,\,\,\,{{\Theta }_h^{\text{*}}}-{{\Theta }_{h}} \nonumber\\ 
 & ={{\Theta }_h^{\text{*}}}-{\mathcal{V}_{h}^{-1}}\left( \textstyle{\sum}_{s={{h}_{0}}}^{h-1}{Z_{s}^{\top}X_{s}^{next}} \right)  \nonumber\\  
 & ={{\Theta }_h^{\text{*}}}-{{\mathcal{V}_{h}^{-1}}}\left( \textstyle{\sum}_{s={{h}_{0}}}^{h-1}{Z_{s}^{\top}{{Z}_{s}}{{\Theta }_s^{\text{*}}}}+\textstyle{\sum}_{s={{h}_{0}}}^{h-1}{Z_{s}^{\top}{{W}_{s}}} \right) \nonumber \\
   & =\Theta _{h}^{*}-{{\mathcal{V}_{h}}^{-1}}\left( \textstyle{\sum}_{s={{h}_{0}}}^{h-1}{Z_{s}^{\top}{{Z}_{s}}(\Theta _{s}^{*}-\Theta _{h}^{*})}-\textstyle{\sum}_{s={{h}_{0}}}^{h-1}{Z_{s}^{\top}{{W}_{s}}} -\textstyle{\sum}_{s={{h}_{0}}}^{h-1}{Z_{s}^{\top}{{Z}_{s}}\Theta _{h}^{*}}  \right)  \nonumber\\  
  & ={\mathcal{V}_{h}^{-1}}\left( \lambda \Theta _{h}^{*}-\textstyle{\sum}_{s={{h}_{0}}}^{h-1}{Z_{s}^{\top}{{W}_{s}}}-\textstyle{\sum}_{s={{h}_{0}}}^{h-1}{Z_{s}^{\top}{{Z}_{s}}(\Theta _{s}^{*}-\Theta _{h}^{*})} \right)  \,.
\end{align}

\noindent
Therefore, based on (\ref{icmleeq17}), the following holds,
\begin{align}
  {{\left\| \Theta _{h}^{*}-{{\Theta }_{h}} \right\|}_{{\mathcal{V}_{h}}}} 
  \overset{(a)}{\le} \ & {{\left\| \lambda I\Theta _{s}^{*} \right\|}_{\mathcal{V}_{h}^{-1}}}+{{\left\| \textstyle{\sum}_{s={{h}_{0}}}^{h-1}{Z_{s}^{\top}{{W}_{s}}} \right\|}_{\mathcal{V}_{h}^{-1}}} +{{\left\| \textstyle{\sum}_{s={{h}_{0}}}^{h-1}{Z_{s}^{\top}{{Z}_{s}}(\Theta _{s}^{*}-\Theta_{h}^{*})} \right\|}_{\mathcal{V}_{h}^{-1}}} \nonumber\\ 
 \le \ & {{\left\| {{\left( \lambda I \right)}^{\frac{1}{2}}}\Theta _{s}^{*} \right\|}_{F}}+{{\left\| \textstyle{\sum}_{s={{h}_{0}}}^{h-1}{Z_{s}^{\top}{{W}_{s}}} \right\|}_{\mathcal{V}_{h}^{-1}}} 
 + {{\left\| \textstyle{\sum}_{s={{h}_{0}}}^{h-1}{Z_{s}^{\top}{{Z}_{s}}(\Theta _{s}^{*}-\Theta _{h}^{*})} \right\|}_{\mathcal{V}_{h}^{-1}}}\nonumber \\ 
 \overset{(b)}{\le} \ & \underbrace{{{\left\| {{\left( \lambda I \right)}^{\frac{1}{2}}} \right\|}_{2}}}_{{\ell_1}}+\underbrace{{{\left\| \textstyle{\sum}_{s={{h}_{0}}}^{h-1}{Z_{s}^{\top}{{W}_{s}}} \right\|}_{\mathcal{V}_{h}^{-1}}}}_{\ell_2} 
 + \underbrace{{{\left\|\textstyle{\sum}_{s={{h}_{0}}}^{h-1}{Z_{s}^{\top}{{Z}_{s}}(\Theta _{s}^{*}-\Theta _{h}^{*})} \right\|}_{\mathcal{V}_{h}^{-1}}}  }_{{\ell_3}} \,.
 \label{ICMLeq18}
\end{align}

\noindent
where $(a)$ follows from triangle inequality and $(b)$ holds from fact that ${{\left\| \mathcal{AB} \right\|}_{F}}\le {{\left\| \mathcal{A} \right\|}_{2}}{{\left\| \mathcal{B} \right\|}_{F}}$ for any two matrices $\mathcal{A}$ and $\mathcal{B}$.

\subsubsection{Proof for bound of $\ell_2$ in Lemma \ref{ICMLthree}}

Let $\left\{ {\mathcal{F}_{t}} \right\}_{t=0}^{\infty }$ be a filtration generated by the random variables $\left\{ {{s}_{t+1}},\,{{a}_{t+1}} \right\}_{t=0}^{\infty }$. Let $\left\{ {{\eta }_{t}} \right\}_{t\ge 1}^{{}}$ be a vector-valued martingale difference process adapted to the filtration $\left\{ {\mathcal{F}_{t}} \right\}_{t=0}^{\infty }$ be a sub-Gaussian random vector, i.e., it satisfies for some $R\ge 0$, the following holds

\begin{equation}
\begin{aligned}
  \mathbb{E}[\exp ({{\alpha }^{\top}}{{\eta }_{t}})|{\mathcal{F}_{t-1}}]\le \exp \left( \frac{{{R}^{2}}{{\left\| \alpha  \right\|}^{2}}}{2} \right),\,\,\,\,\,\forall t\ge 1,\,\,\forall \alpha \in {\mathbb{R}^{n}} \,.
 \label{eq33}
\end{aligned}
\end{equation}

Let $V_t={{ \sum\limits_{i={{1}_{}}}^{t}{Z_{i}^\top{{Z}_{i}}} }}$, $\mathcal{V}_t={{ V_t+\lambda I_m} }$, $S_t=\sum\limits_{s={{1}_{}}}^{t}{Z_{i}^{\top}{{W}_{i}}}$ and $d=m+n$. For any $0<\delta\le 1$, with probability at least $1-\delta$, uniformly over all $t\ge 1$, it holds that 

\begin{equation}
\begin{aligned}
{{\left\| \mathcal{V}_{t}^{-1/2}{{S}_{t}} \right\|}_{F}}\le R\sqrt{2\ln \left( \frac{1}{\delta } \right)+n\ln \frac{\det ({\mathcal{V}_{t}})}{\det (\lambda {{I}_{d}})}}\,.
 \label{eq33}
\end{aligned}
\end{equation}

\begin{proof}

For any $\gamma \in \mathbb{R}^{d\times n}$ and $t\ge 0$, let us define

\begin{equation}
\begin{aligned}
M_{t}^{\gamma }=\exp \left( \frac{1}{R}tr\left( {{\gamma }^{\top}}{{S}_{t}} \right)-\frac{1}{2}tr({{\gamma }^{\top}}{{S}_{t}}\gamma ) \right)\,.
 \label{eq33}
\end{aligned}
\end{equation}

From where we derive that $M_{t}^{\gamma}=\prod\limits_{i=1}^{t}{D_{i}^{\gamma }}$, where  

\begin{align}
  & D_{i}^{\gamma }=\exp \left( \frac{1}{R}tr\left( {{\gamma }^{\top}}Z_{i}^{\top}{{W}_{i}} \right)-\frac{1}{2}tr\left( {{\gamma }^{\top}}Z_{i}^{\top}{{Z}_{i}}\gamma  \right) \right) \nonumber\\ 
 & \,\,\,\,\,\,\,=\exp \left( \frac{1}{R}tr\left( {{W}_{i}}{{\gamma }^{\top}}Z_{i}^{\top} \right)-\frac{1}{2}tr\left( {{Z}_{i}}\gamma {{\gamma }^{\top}}Z_{i}^{\top} \right) \right)\nonumber \\ 
 & \,\,\,\,\,\,\,=\exp \left( \frac{{{W}_{i}}{{\gamma }^{\top}}Z_{i}^{\top}}{R}-\frac{1}{2}{{\left\| {{\gamma }^{\top}}Z_{i}^{\top} \right\|}^{2}} \right) \,.
\end{align}

Note that $M_t^{\gamma}\ge 0$ and $D_t^{\gamma}$ is $\mathcal{F}_t$ measurable, as is $M_t^{\gamma}\ge 0$. Further, due to the conditional sub-Gaussian property, it holds that $\mathbb{E}\left[ D_{t}^{\gamma }|{\mathcal{F}_{t-1}} \right]\le 1$, and thus that $\mathbb{E}\left[ M_{t}^{\gamma}|{\mathcal{F}_{t-1}} \right]\le M_{t-1}^{\gamma}$. Therefore, $\left\{ M_{t}^{\gamma } \right\}_{t=0}^{\infty }$ is a super-martingale w.r.t the filtration $\left\{ \mathcal{F}_{t}^{\gamma } \right\}_{t=0}^{\infty }$ satisfying $\mathbb{E}\left[ M_{t}^{\gamma} \right]\le 1$.

Let $\tau$ be a stopping time w.r.t filtration $\left\{ \mathcal{F}_{t}^{\gamma } \right\}_{t=0}^{\infty }$. By the convergence theorem for non-negative super-martingales, $M_{\infty}^{\gamma}=lim_{t\to \infty} M_t^g$ is almost surely well-defined. Thus $M_{\tau}^{\gamma}$ is well-defined as well, irrespective of whether $\tau$ is finite or not. Let $Q_t^{\gamma}=M_{min{\{\tau,t\}}}^{\gamma}$ be a stop version of $\left\{ {M}_{t}^{\gamma } \right\}_{t}$. By Fatou's lemma,

\begin{equation}
\begin{aligned}
\mathbb{E}\left[ M_{t}^{\gamma } \right]=\mathbb{E}\left[ \underset{t\to \infty }{\mathop{\lim \,\inf }}\,Q_{t}^{\gamma } \right]\le \underset{t\to \infty }{\mathop{\lim \,\inf }}\,\mathbb{E}\left[ Q_{t}^{\gamma } \right]=\underset{t\to \infty }{\mathop{\lim \,\inf }}\,\mathbb{E}\left[ M_{min{\{\tau,t\}}}^{\gamma} \right]\le 1\,.
\end{aligned}
\end{equation}

since the stopped super-martingale $\left\{ M_{min{\{\tau,t\}}}^{\gamma} \ \right\}_{t}$ is also super-martingale.

Let $\mathcal{F}_{\infty}$ be $\sigma$-algebra generated by $\left\{ \mathcal{F}_{t}^{\gamma } \right\}_{t=0}^{\infty }$, and $\Gamma \in \mathbb{R}^{d\times n}$ be a random matrix with its entries being i.i.d. according to $\mathcal{N}(0,\lambda^{-1})$ independent of $\mathcal{F}_{\infty}$. Define a mixture of super-martingale $M_t=\mathbb{E}\left[ M_{t}^{\tau}|{\mathcal{F}_{\infty}} \right]$, and it is immediate to see that $\left\{ {M}_{t} \right\}_{t}$ is also a non-negative super-martingale w.r.t. the filtration $\left\{ \mathcal{F}_{t} \right\}_{t}$. Hence, by similar argument, $M_{\tau}$ is well-defined and following holds

\begin{equation}
\begin{aligned}
\mathbb{E}\left[ M_{t} \right]=\mathbb{E}\left[ M_{t}^{\Gamma } \right]= \,\mathbb{E}\left[ \mathbb{E}\left[ M_{t}^{\tau}|{\mathcal{F}_{\infty}} \right] \right]\le\,\mathbb{E}\left[ 1 \right]= 1\\ 
\end{aligned}
\end{equation}

Now let us start to compute $M_t$. To this end, we define $\mathcal{S}_t=\frac{S_t}{R}$ and $V=\lambda I_d$. The joint probability density function of $\Gamma$ is given by 

\begin{equation}
\begin{aligned}
f(\gamma )={{\left( \sqrt{\lambda /2\pi } \right)}^{dn}}\exp \left( -\frac{\lambda }{2}\sum\limits_{i=1}^{d}{\sum\limits_{j=1}^{n}{\gamma _{i,j}^{2}}} \right)={{\left( \frac{\det {{\left( V \right)}^{\frac{1}{2}}}}{{{(2\pi )}^{\frac{d}{2}}}} \right)}^{n}}\exp \left( -\frac{1}{2}tr\left( {{\gamma }^{\top}}V\gamma  \right) \right)\,.
\end{aligned}
\end{equation}

\noindent
For any p.d. matrix $M$, define $c(M)={{\left( \frac{\det {{\left( M \right)}^{\frac{1}{2}}}}{{{\left( 2\pi  \right)}^{\frac{m}{2}}}} \right)}^{n}}$. Then, we have \footnote{We make $tr$ and Trace interchangeable.}

\begin{align}
   {{M}_{t}}&=\int\limits_{{\mathbb{R}^{m\times n}}}{\exp \left( tr\left( {{\gamma }^{\top}}{\mathcal{S}_{t}} \right)-\frac{1}{2}tr\left( {{\gamma }^{\top}}{{V}_{t}}\gamma  \right) \right)}f(\gamma )d\gamma  \nonumber\\
   & =\int\limits_{{\mathbb{R}^{m\times n}}}{\exp \left( -\frac{1}{2}tr\left( {{\left( \gamma -V_{t}^{-1}{\mathcal{S}_{t}} \right)}^{\top}}{{V}_{t}}\left( \gamma -V_{t}^{-1}{\mathcal{S}_{t}} \right) \right)+\frac{1}{2}tr\left( \mathcal{S}_{t}^{\top}V_{t}^{-1}{\mathcal{S}_{t}} \right) \right)}f(\gamma )d\gamma  \nonumber\\
   &=c(V)\exp \left( \frac{1}{2}tr\left( \mathcal{S}_{t}^{\top}V_{t}^{-1}{\mathcal{S}_{t}} \right) \right)\int\limits_{{\mathbb{R}^{m\times n}}}{\exp \left( -\frac{1}{2}\left\{ tr\left( {{\left( \gamma -V_{t}^{-1}{\mathcal{S}_{t}} \right)}^{\top}}{{V}_{t}}\left( \gamma -V_{t}^{-1}{\mathcal{S}_{t}} \right) \right)+\frac{1}{2}tr\left( {{\gamma }^{\top}}V\gamma  \right) \right\} \right)}d\gamma \nonumber \\
   &=c(V)\exp \left( \frac{1}{2}tr\left( \mathcal{S}_{t}^{\top}\mathcal{V}_{t}^{-1}{\mathcal{S}_{t}} \right) \right)\int\limits_{{\mathbb{R}^{m\times n}}}{\exp \left( -\frac{1}{2}tr\left( {{\left( \gamma -\mathcal{V}_{t}^{-1}{\mathcal{S}_{t}} \right)}^{\top}}{\mathcal{V}_{t}}\left( \gamma -\mathcal{V}_{t}^{-1}{\mathcal{S}_{t}} \right) \right) \right)}d\gamma \,.
\end{align}

\noindent
where in the last step we use the fact that $\mathcal{V}_t=V_t+V$ and

\begin{align}
  & \,\,\,\,\,\,\,tr\left( {{\left( \gamma -V_{t}^{-1}{\mathcal{S}_{t}} \right)}^{\top}}{{V}_{t}}\left( \gamma -V_{t}^{-1}{\mathcal{S}_{t}} \right) \right)+tr\left( {{\gamma }^{\top}}V\gamma  \right) \nonumber\\ 
 & \,=tr\left( {{\left( \gamma -\mathcal{V}_{t}^{-1}{\mathcal{S}_{t}} \right)}^{\top}}{\mathcal{V}_{t}}\left( \gamma -\mathcal{V}_{t}^{-1}{{S}_{t}} \right) \right)+tr\left( \mathcal{S}_{t}^{\top}V_{t}^{-1}{\mathcal{S}_{t}} \right)-tr\left( \mathcal{S}_{t}^{\top}\mathcal{V}_{t}^{-1}{\mathcal{S}_{t}} \right) \,.
\end{align}

Let $P(1), ..., P(n)$ denote the columns a $m$-by-$n$ matrix $P$, and $A_t=\mathcal{V}_t^{-1}\mathcal{S}_t$, then

\begin{align}
tr\left( {{\left( \gamma -\mathcal{V}_{t}^{-1}{{S}_{t}} \right)}^{\top}}{\mathcal{V}_{t}}\left( \gamma -\mathcal{V}_{t}^{-1}{{S}_{t}} \right) \right)
=\ & \sum\limits_{i=1}^{n}{{{\left( \gamma (i)-{{A}_{t}}(i) \right)}^{\top}}{\mathcal{V}_{t}}}\left( \gamma (i)-{{A}_{t}}(i) \right) \nonumber \\
=\ & \sum\limits_{i=1}^{n}{\left\| \gamma (i)-{{A}_{t}}(i) \right\|_{\mathcal{{V}_{t}}}^{2}} \,,
\end{align}

which yields
\begin{align}
  & {{M}_{t}}=c(V)\exp \left( \frac{1}{2}tr\left( \mathcal{S}_{t}^{\top}V_{t}^{-1}{\mathcal{S}_{t}} \right) \right)\int\limits_{{\mathbb{R}^{d\times n}}}{\prod\limits_{i=1}^{n}{\exp \left( -\frac{1}{2}\left\| \gamma (i)-{{A}_{t}}(i) \right\|_{{\mathcal{V}_{t}}}^{2} \right)}}d\gamma  \nonumber\\ 
 & \,\,\,\,\,\,\,\,=c(V)\exp \left( \frac{1}{2}tr\left( \mathcal{S}_{t}^{\top}\mathcal{V}_{t}^{-1}{\mathcal{S}_{t}} \right) \right)\prod\limits_{i=1}^{n}{\int\limits_{{\mathbb{R}^{d}}}{\exp \left( -\frac{1}{2}\left\| \gamma (i)-{{A}_{t}}(i) \right\|_{{\mathcal{V}_{t}}}^{2} \right)}}d\gamma (i) \nonumber\\ 
 & \,\,\,\,\,\,\,\,=c(V)\exp \left( \frac{1}{2}tr\left( \mathcal{S}_{t}^{\top}\mathcal{V}_{t}^{-1}{\mathcal{S}_{t}} \right) \right)\prod\limits_{i=1}^{n}{\frac{{{\left( 2\pi  \right)}^{\frac{m}{2}}}}{\det {{\left( {\mathcal{V}_{t}} \right)}^{\frac{1}{2}}}}} \nonumber\\ 
 & \,\,\,\,\,\,\,\,={{\left( \frac{\det (\lambda {{I}_{d}})}{\det (\lambda {{I}_{d}}+{{V}_{t}})} \right)}^{\frac{n}{2}}}\exp \left( \frac{1}{2}tr\left( \mathcal{S}_{t}^{\top}\mathcal{V}_{t}^{-1}{\mathcal{S}_{t}} \right) \right) \,.
\end{align}

In the above steps, we rely in the fact that $\int\limits_{{\mathbb{R}^{d}}}{\exp \left( -\frac{1}{2}\left\| \gamma (i)-{{A}_{t}}(i) \right\|_{{{V}_{t}}}^{2} \right)}=\sqrt{\frac{{{\left( 2\pi  \right)}^{d}}}{\det ({\mathcal{V}_{t}})}}$ since $\mathcal{V}_t \in \mathbb{R}^{d\times d}$ is a positive definite matrix.

Now for any $0< \delta\le 1 $, the following holds by Markov's inequality

\begin{align}
  \mathbb{P}\left[ tr\left( \mathcal{S}_{\tau }^{\top}\mathcal{V}_{\tau }^{-1}{\mathcal{S}_{\tau }} \right)>2\log \left( \frac{\det {{\left( {{V}_{\tau }}+\lambda {{I}_{d}} \right)}^{\frac{n}{2}}}}{\delta \det {{\left( \lambda {{I}_{d}} \right)}^{\frac{n}{2}}}} \right) \right]
  =\ & \mathbb{P}\left[ \frac{\exp \left( \frac{1}{2}tr\left( \mathcal{S}_{\tau }^{\top}\mathcal{V}_{\tau }^{-1}{\mathcal{S}_{\tau }} \right) \right)}{\frac{1}{\delta }{{\left( \frac{\det \left( {{V}_{\tau }}+\lambda {{I}_{d}} \right)}{\delta \det \left( \lambda {{I}_{d}} \right)} \right)}^{\frac{n}{2}}}}>1 \right] \nonumber\\ 
 \le \ &\delta \mathbb{E}\left[ {{M}_{\tau }} \right]\le \delta  \,.
\end{align}

To complete the proof, we define $\tau$ as 

$\tau =\min \left\{ t\ge 0:\,\,tr\left( \mathcal{S}_{\tau }^{\top}\mathcal{V}_{\tau }^{-1}{\mathcal{S}_{\tau }} \right)>2\log \left( \frac{\det {{\left( {{V}_{\tau }}+\lambda {{I}_{d}} \right)}^{\frac{n}{2}}}}{\delta \det {{\left( \lambda {{I}_{d}} \right)}^{\frac{n}{2}}}} \right) \right\}$,
where $\min{\{ \emptyset\}}=\infty$ by convention. Clearly, $\tau$ is a random stopping time and 

\begin{equation}
\begin{aligned}
  & \,\,\,\,\mathbb{P}\left[ \forall t\ge0, \,\,\,\, tr\left( \mathcal{S}_{\tau }^{\top}\mathcal{V}_{\tau }^{-1}{\mathcal{S}_{\tau }} \right)>2\log \left( \frac{\det {{\left( {{V}_{\tau }}+\lambda {{I}_{d}} \right)}^{\frac{n}{2}}}}{\delta \det {{\left( \lambda {{I}_{d}} \right)}^{\frac{n}{2}}}} \right) \right] \\ 
 & =\mathbb{P}\left[ \tau <\infty  \right]\le \mathbb{P}\left[ tr\left( \mathcal{S}_{\tau }^{\top}\mathcal{V}_{\tau }^{-1}{\mathcal{S}_{\tau }} \right)>2\log \left( \frac{\det {{\left( {{V}_{\tau }}+\lambda {{I}_{d}} \right)}^{\frac{n}{2}}}}{\delta \det {{\left( \lambda {{I}_{d}} \right)}^{\frac{n}{2}}}} \right) \right]\le \delta \\ 
 \label{ICMLMATRIX}
\end{aligned}
\end{equation}

The proof concludes due to the fact that $tr\left( \mathcal{S}_{t }^{\top}\mathcal{V}_{t }^{-1}{\mathcal{S}_{t }} \right)=\frac{1}{R^2}tr\left( {S}_{t }^{\top}\mathcal{V}_{t }^{-1}{{S}_{t }} \right)={{\left( \frac{1}{R}{{\left\| \mathcal{V}_{t}^{-1/2}{{S}_{t}} \right\|}_{F}} \right)}^{2}}$.

\end{proof}

\subsubsection{Proof for bound of $\ell_2$ in Lemma \ref{ICMLthree2}}

\begin{proof}
The key difference when Lemma \ref{ICMLthree} is compared to $\ell_2$ is a the bound on ${{\left\| \textstyle{\sum}_{s={{h}_{0}}}^{h-1}{Z_{s}^{\top}{{W}_{s}}} \right\|}_{\Tilde{\mathcal{V}}_{h}^{-1}}}$, where $h_0=\text{max}(1,h-\mathcal{W}) $. Following the trick used in \cite{Russac2019} to handle the information loss during the sliding window, we use the union bound instead of the 'stopping time' trick to get 
${{\left\| \mathcal{V}_{t}^{-1/2}{{S}_{t}} \right\|}_{F}}\le R\sqrt{2\ln \left( \frac{H}{\delta } \right)+n\ln \frac{\det ({\mathcal{V}_{t}})}{\det (\lambda {{I}_{d}})}}$.
\end{proof}

\subsection{Proof for bound of $\ell_3$ }

\subsubsection{Proof for bound of $\ell_3$ in Lemma \ref{ICMLthree} }
\label{newnewnew}
Now we begin to prove the bound of $\ell_3$. Firstly, we bound $\ell_3$ in Lemma \ref{ICMLthree} (using R strategy).
We first propose the following lemma.

\begin{lemma} For any $h\in [H]$, we have 

\[{{\left\| {{\left( \sum\limits_{s={{h}_{0}}}^{h-1}{Z_{s}^{\top}{{Z}_{s}}+\lambda I} \right)}^{-1}}\left( \sum\limits_{s={{h}_{0}}}^{h-1}{Z_{s}^{\top}{{Z}_{s}}(\Theta _{s}^{*}-\Theta _{h}^{*})} \right) \right\|}_{F}} \le  (1+\gamma)\sqrt{\frac{{{L}(m+n)}}{\lambda }}\mathcal{B}_H \,,\]
\label{ICMLTV}

where the ${L}$ is the size of the restart epoch.
\end{lemma}

\begin{proof}
See Appendix \ref{ICMLAPPAX11}.
\end{proof}

Now we are ready to prove the bound of $\ell_3$ in Lemma \ref{ICMLthree} .

\begin{proof} For any $h\in [H]$, one has, 

\begin{align}
 & \ \left\| \left( \sum\limits_{s={{h}_{0}}}^{h-1}{Z_{s}^{\top}{{Z}_{s}}(\Theta _{s}^{*}-\Theta _{h}^{*})} \right) \right\|_{\mathcal{V}_h^{-1}}^{2} \nonumber\\ 
 = \ & \left\| {{\left( \sum\limits_{s={{h}_{0}}}^{h-1}{Z_{s}^{\top}{{Z}_{s}}+\lambda I} \right)}^{-1}}\left( \sum\limits_{s={{h}_{0}}}^{h-1}{Z_{s}^{\top}{{Z}_{s}}(\Theta _{s}^{*}-\Theta _{h}^{*})} \right) \right\|_{\left( \sum\limits_{s={{h}_{0}}}^{h-1}{Z_{s}^{\top}{{Z}_{s}}+\lambda I} \right)}^{2} \nonumber\\ 
 \le \ & {{\lambda }_{\max }}\left( \sum\limits_{s={{h}_{0}}}^{h-1}{Z_{s}^{\top}{{Z}_{s}}+\lambda I} \right)\left\| {{\left( \sum\limits_{s={{h}_{0}}}^{h-1}{Z_{s}^{\top}{{Z}_{s}}+\lambda I} \right)}^{-1}}\left( \sum\limits_{s={{h}_{0}}}^{h-1}{Z_{s}^{\top}{{Z}_{s}}(\Theta _{s}^{*}-\Theta _{h}^{*})} \right) \right\|_{F}^{2} \,.   \label{App2}
\end{align}

Then putting Corollary \ref{ICMLBKBoun} and Lemma \ref{ICMLTV} into (\ref{App2}), we finish the proof for bound of $\ell_3$ in Lemma \ref{ICMLthree}. 
\end{proof}

\subsubsection{Proof for bound of $\ell_3$ in Lemma \ref{ICMLthree2} }

From the mathematical conduction of Section \ref{newnewnew}, one can easily verify that the results in Section \ref{newnewnew} can be directly applied to Lemma \ref{ICMLthree2} by replacing $L$ with sliding window size $\mathcal{W}$.

\section{Proof for Lemma \ref{ICMLLL5} in Section \ref{ICMLADE}}
\label{ICMLT4.7}

\begin{proof}
 We begin with the following decomposition of the dynamic regret.

\begin{align}
\mathcal{R}(K)=\ &\sum\limits_{k=1}^{K}{J_{1}^{{{\pi }_{k}}}({{\Theta }_{*}},{{x}_{k,1}})}-J_{1}^{*}({{\Theta }_{*}},{{x}_{k,1}}) \nonumber\\ 
 \le \ & \sum\limits_{k=1}^{K}{J_{1}^{{{\pi }_{k}}}({{\Theta }_{*}},{{x}_{k,1}})}-J_{1}^{*}(\widetilde{\Theta}_k,{{x}_{k,1}}) \nonumber\\ 
=\ &\sum\limits_{k=1}^{K}\Gamma_{k,1} \,,
\label{eq71}
\end{align}
where the inequality holds due to optimistic algorithm (i.e., (\ref{icmleeq15})) under the event $\mathcal{\varepsilon}_K(\delta)$. To bound this, we first investigate $\Gamma_{k,h}$. Note that the action $u_{k,h}$ under $\pi_k$ is the same as that under an optimal policy when the true dynamic is $\Tilde{\Theta}$, hence

\begin{align}
\Gamma_{k,h}
=\ &||{{x}_{k,h}}|{{|}_{{{Q}_{}}}}+||{{u}_{k,h}}|{{|}_{{{R}_{}}}}+\mathbb{E}[J_{h+1}^{{{\pi }_{k}}}({{\Theta }_{*}},{{x}_{k,h+1}})|{\mathcal{F}_{k,h}}] \nonumber\\ 
 & \ -||{{x}_{k,h}}|{{|}_{{{Q}_{}}}}-||{{u}_{k,h}}|{{|}_{{{R}_{}}}}-\sum\limits_{h'=h+1}^{H}{\mathbb{E}[||{{w}_{h'}}|{{|}_{{{P}_{h'+1}}({{\widetilde{\Theta}}_{k}})}}]} 
  -\mathbb{E}[||{{x}_{k,h+1}}|{{|}_{{{P}_{h'+1}}}}|{\mathcal{F}_{k,h}}] \,.
  \label{eq7}
\end{align}

Denote ${{\Gamma }_{k,h}}={{\Delta }_{k,h}}+J_{h+1}^{{{\pi }_{k}}}({{\Theta }_{*}},{{x}_{k,h+1}})-{{\psi }_{k,h+1}}-\mathbb{E}[||\widetilde{\Theta }_{k}^{\top}{{z}_{k,h}}+{{w}_{k,h}}|{{|}_{{{P}_{h'+1}}}}|{\mathcal{F}_{k,h}}]$, where ${{\psi }_{k,h+1}}=\sum\limits_{h'=h+1}^{H}{\mathbb{E}[||{{w}_{h'}}|{{|}_{{{P}_{h'+1}}({{\widetilde{\Theta}}_{k}})}}]}$ and ${{\Delta }_{k,h}}\text{=}\mathbb{E}[J_{h+1}^{{{\pi }_{k}}}({{\Theta }_{*}},{{x}_{k,h+1}})|{{F}_{k,h}}]-J_{h+1}^{{{\pi }_{k}}}({{\Theta }_{*}},{{x}_{k,h+1}})$. Now, we rewrite $\Gamma_{k,h}$ as follows,

\begin{align}
  {{\Gamma }_{k,h}} 
 \overset{(a)}{=} \ &{{\Delta }_{k,h}}+J_{h+1}^{{{\pi }_{k}}}({{\Theta }_{*}},{{x}_{k,h+1}})-{{\psi }_{k,h+1}}-||\widetilde{\Theta }_{k,h}^{\top}{{z}_{k,h}}|{{|}_{{{\widetilde{P}}_{k,h+1}}}} \nonumber\\ 
 & \ -\mathbb{E}\left[ ||{{w}_{k,h}}|{{|}_{{{\widetilde{P}}_{k,h+1}}}}|{\mathcal{F}_{k,h}} \right] \nonumber\\ 
 = \ &{{\Delta }_{k,h}}+J_{h+1}^{{{\pi }_{k}}}({{\Theta }_{*}},{{x}_{k,h+1}})-{{\psi }_{k,h+1}}-||\widetilde{\Theta }_{k,h}^{\top}{{z}_{k,h}}|{{|}_{{{\widetilde{P}}_{k,h+1}}}} \nonumber\\ 
 & \ -\mathbb{E}\left[ ||{{x}_{k,h+1}}-\Theta _{*}^{\top}{{z}_{k,h}}|{{|}_{{{\widetilde{P}}_{k,h+1}}}}|{\mathcal{F}_{k,h}} \right] \nonumber\\ 
 \overset{(b)}{=}\ &{{\Delta }_{k,h}}+J_{h+1}^{{{\pi }_{k}}}({{\Theta }_{*}},{{x}_{k,h+1}})-{{\psi }_{k,h+1}}-||\widetilde{\Theta }_{k,h}^{\top}{{z}_{k,h}}|{{|}_{{{\widetilde{P}}_{k,h+1}}}} \nonumber\\ 
 & \ -\mathbb{E}\left[ ||{{x}_{k,h+1}}|{{|}_{{{\widetilde{P}}_{k,h+1}}}}|{{F}_{k,h}} \right]+||\Theta _{*}^{\top}{{z}_{k,h}}|{{|}_{{{\widetilde{P}}_{k,h+1}}}} \,,
\end{align}
where in $(a)$ and $(b)$, we use the independence and mean zero properties of $w_{k,h}$. Notice that 
${{\psi }_{k,h+1}}=J_{h}^{*}({{\widetilde{\Theta }}_{k,h+1}},{{x}_{k,h+1}})-||{{x}_{k,h+1}}|{{|}_{{{\widetilde{P}}_{k,h+1}}}}$, then, (\ref{eq7}) can be expressed as

Then,

\begin{align}
 {{\Gamma }_{k,h}}  
 = \ &{{\Delta }_{k,h}}+J_{h+1}^{{{\pi }_{k}}}({{\Theta }_{*}},{{x}_{k,h+1}})-J_{h}^{*}({{\widetilde{\Theta }}_{k,h+1}},{{x}_{k,h+1}})+||{{x}_{k,h+1}}|{{|}_{{{\widetilde{P}}_{k,h+1}}}} \nonumber\\ 
 & \ -||\widetilde{\Theta }_{k,h}^{\top}{{z}_{k,h}}|{{|}_{{{\widetilde{P}}_{k,h+1}}}}-E\left[ ||{{x}_{k,h+1}}|{{|}_{{{\widetilde{P}}_{k,h+1}}}}|{{F}_{k,h}} \right]+||\Theta _{*}^{\top}{{z}_{k,h}}|{{|}_{{{\widetilde{P}}_{k,h+1}}}} \,.
 \label{eeq26}
\end{align}

Due to the fact that the cost of $H+1$ step and beyond are $0$, summarize (\ref{eeq26}) yields

\begin{equation}
\begin{aligned}
  & \mathcal{R}(K)\le \sum\limits_{k=1}^{K}{\sum\limits_{h=1}^{H-1}{\varsigma}_{h,k}} \,.
  \label{eeeq23}
\end{aligned}
\end{equation}

\noindent
Thus, we finish the proof of Lemma \ref{ICMLLL5}.
\end{proof}

\section{Proof of Lemma \ref{ICMLL1} in Section \ref{ICMLADE}}
\label{ICMLT4.8}

\begin{proof}

We begin with the following decomposition of the dynamic regret.

For the sake of convenient, denote $I_{k,h}=||{{x}_{k,h+1}}|{{|}_{{{\widetilde{P}}_{k,h+1}}}}-\mathbb{E}\left[ ||{{x}_{k,h+1}}|{{|}_{{{\widetilde{P}}_{k,h+1}}}}|{\mathcal{F}_{k,h}}\right]$
, from where one can verify that the sequence of term $\Delta _{k,h}$ and $I_1$ form a martingale difference sequence. Meanwhile 

\[E\left[ {{I}_{1}}|{\mathcal{F}_{k,h}} \right]=0,\,\,\,E\left[ {{\Delta }_{k,h}}|{\mathcal{F}_{k,h}} \right]=0\,,\]

holds since $\mathcal{F}_{k,h}$ is all randomness before the step $(k,h)$. In order to bound the term $I_1$ and ${\Delta }_{k,h}$, we resort to Corollary~\ref{ICMLBK0} as well as Assumption~\ref{A1}. From where the following inequalities hold

\begin{align}
  \left| {{I}_{2}} \right|
  = \ &\left[ ||{{x}_{k,h+1}}|{{|}_{{{\widetilde{P}}_{k,h+1}}}}  
 -\mathbb{E}\left[ ||{{x}_{k,h+1}}|{{|}_{{{\widetilde{P}}_{k,h+1}}}}|{\mathcal{F}_{k,h}}\right] \right] \nonumber\\ 
 \le \ &2 \upsilon\,,
 \label{28}
\end{align}
and to bound $\Delta _{k,h}$, we bound it backwards by using Assumption \ref{A1}, first notice that

\begin{equation}
\begin{aligned}
  & \,\,\,\,\left| J_{H}^{{{\pi }_{k}}}({{\Theta }_{*}},{{x}_{k,H}}) \right|  ={{\left\| {{x}_{k,H}} \right\|}_{{{Q}_{}}}}+{{\left\| {{u}_{k,H}} \right\|}_{{{R}_{}}}}=(1+{{\gamma }^{2}}) \upsilon \,. 
\label{29}
\end{aligned}
\end{equation}

Thus, for $h\in [H]$, we have
\begin{equation}
\begin{aligned}
  &\left| J_{h}^{{{\pi }_{k}}}({{\Theta }_{*}},{{x}_{k,h}}) \right| ={{\left\| {{x}_{k,h}} \right\|}_{{{Q}_{}}}}+{{\left\| {{u}_{k,h}} \right\|}_{{{R}_{}}}}+\left| \mathbb{E}\left[ J_{h+1}^{{{\pi }_{k}}}({{\Theta }_{*}},{{x}_{k,h+1}})|{\mathcal{F}_{k,h}} \right] \right| \le H(1+{{\gamma }^{2}}) \upsilon\,. 
\label{29}
\end{aligned}
\end{equation}

Finally, we apply the Azuma-Hoeffding inequality (Lemma \ref{Azu-Hoe}) to (\ref{28}) and (\ref{29})

\begin{equation}
\begin{aligned}
  \sum\limits_{k=1}^{K}{\sum\limits_{h=1}^{H}{{{I}_{1}}}}  
  \le \ & 4C\sqrt{KH\ln \frac{1}{\delta }} \le O\,\,\left( \sqrt{KH\ln \frac{1}{\delta }} \right)\,, \nonumber\\
  \sum\limits_{k=1}^{K}\sum\limits_{h=1}^{H}{{{\Delta }_{k,h}}}\le\ & 2H(1+{{\gamma }^{2}})C\sqrt{KH\ln \frac{1}{\delta }}\le O\left( \sqrt{K{{H}^{3}}\ln \frac{1}{\delta }} \right) \,.
\end{aligned}
\end{equation}

Now we investigate the third term in Lemma \ref{ICMLL1}. Note that
\begin{align}
  & \sum\limits_{k=1}^{K}{\sum\limits_{h=1}^{H}{||\Theta _{*}^{\top}{{z}_{k,h}}||_{{{\widetilde{P}}_{k,h+1}}}^{2}-}}||{{\widetilde{\Theta }}^{\top}}{{z}_{k,h}}||_{{{\widetilde{P}}_{k,h+1}}}^{2} \nonumber\\ 
 \le \ & \sum\limits_{k=1}^{K}{\sum\limits_{h=1}^{H}{\left| ||\Theta _{*}^{\top}{{z}_{k,h}}||_{{{\widetilde{P}}_{k,h+1}}}^{{}}-||{{\widetilde{\Theta }}^{\top}}{{z}_{k,h}}||_{{{\widetilde{P}}_{k,h+1}}}^{{}} \right|}} \nonumber\\ 
 = \ & \sum\limits_{k=1}^{K}{\sum\limits_{h=1}^{H}{\left| \left( ||\widetilde{P}_{k,h+1}^{1/2}\Theta _{*}^{\top}{{z}_{k,h}}|{{|}_{2}}-||\widetilde{P}_{k,h+1}^{1/2}{{\widetilde{\Theta }}^{\top}}{{z}_{k,h}}|{{|}_{2}} \right)      \left( ||\widetilde{P}_{k,h+1}^{1/2}\Theta _{*}^{\top}{{z}_{k,h}}|{{|}_{2}}+||\widetilde{P}_{k,h+1}^{1/2}{{\widetilde{\Theta }}^{\top}}{{z}_{k,h}}|{{|}_{2}} \right)      \right|}} \nonumber\\ 
 \le \ & \underbrace{{{\left[ \sum\limits_{k=1}^{K}{{{\sum\limits_{h=1}^{H}{\left( ||\widetilde{P}_{k,h+1}^{1/2}\Theta _{*}^{\top}{{z}_{k,h}}|{{|}_{2}}+||\widetilde{P}_{k,h+1}^{1/2}{{\widetilde{\Theta }}^{\top}}{{z}_{k,h}}|{{|}_{2}} \right)}}^{2}}} \right]}^{1/2}}}_{\mathcal{\ell_A}} \nonumber\\ \times
 & \ \underbrace{ {{\left[ \sum\limits_{k=1}^{K}{\sum\limits_{h=1}^{H}{{{\left( ||\widetilde{P}_{k,h+1}^{1/2}\Theta _{*}^{\top}{{z}_{k,h}}|{{|}_{2}}-||\widetilde{P}_{k,h+1}^{1/2}{{\widetilde{\Theta }}^{\top}}{{z}_{k,h}}|{{|}_{2}} \right)}^{2}}}} \right]}^{1/2}}}_{\mathcal{\ell_B}} \,.
 \label{nips25}
\end{align}

We bound the ${\mathcal{\ell_A}}$ and ${\mathcal{\ell_B}}$ respectively. We first bound the  ${\mathcal{\ell_A}}$ as follows

\begin{align}
  & {{\left[ \sum\limits_{k=1}^{K}{{{\sum\limits_{h=1}^{H}{\left( ||\widetilde{P}_{k,h+1}^{1/2}\Theta _{*}^{\top}{{z}_{k,h}}|{{|}_{2}}+||\widetilde{P}_{k,h+1}^{1/2}{{\widetilde{\Theta }}^{\top}}{{z}_{k,h}}|{{|}_{2}} \right)}}^{2}}} \right]}^{1/2}} \nonumber\\ 
 \le \ &{{\left[ \sum\limits_{k=1}^{K}{{{\sum\limits_{h=1}^{H}{\left( ||\widetilde{P}_{k,h+1}^{1/2}\Theta _{*}^{\top}{{z}_{k,h}}|{{|}_{2}}+||\widetilde{P}_{k,h+1}^{1/2}{{\widetilde{\Theta }}^{\top}}{{z}_{k,h}}|{{|}_{2}} \right)}}^{2}}} \right]}^{1/2}} \nonumber \\ 
 =\ &{{\left[ \sum\limits_{k=1}^{K}{{{\sum\limits_{h=1}^{H}{\left( ||\widetilde{P}_{k,h+1}^{1/2}{{z}_{k,h+1}}|{{|}_{2}}+||\widetilde{P}_{k,h+1}^{1/2}{{\widetilde{z}}_{k,h+1}}|{{|}_{2}} \right)}}^{2}}} \right]}^{1/2}} \nonumber \\ 
 \le \ &{{\left[ \sum\limits_{k=1}^{K}{{{\sum\limits_{h=1}^{H}{\left( ||\widetilde{P}_{k,h+1}^{1/2}|{{|}_{2}}||{{z}_{k,h+1}}|{{|}_{2}}+||\widetilde{P}_{k,h+1}^{1/2}|{{|}_{2}}{{\widetilde{z}}_{k,h+1}}|{{|}_{2}} \right)}}^{2}}} \right]}^{1/2}} \nonumber\\ 
 \le \ & {{\left[ \sum\limits_{k=1}^{K}{{{\sum\limits_{h=1}^{H}{\left( 2D(1+\gamma ) \right)}}^{2}}} \right]}^{1/2}} \nonumber\\ 
 = \ &2D\sqrt{KH}(1+\gamma ) \,.
 \label{ICMLTHIR}
\end{align}

Then, under the Assumption \ref{A1} and event $\mathcal{\varepsilon}_K(\delta)$, we have

\begin{align}
  {{\left( {{\left\| \widetilde{P}_{k,h+1}^{1/2}\Theta _{*}^{\top}{{z}_{k,h}} \right\|}_{2}}-{{\left\| \widetilde{P}_{k,h+1}^{1/2}\widetilde{\Theta }_{k,h}^{\top}{{z}_{k,h}} \right\|}_{2}} \right)}^{2}} 
 \le \ &\left\| \widetilde{P}_{k,h+1}^{1/2}{{(\Theta _{*}^{{}}-\widetilde{\Theta }_{k,h}^{{}})}^{\top}}{{z}_{k,h}} \right\|_{2}^{2}\nonumber \\ 
 \le \ & {{D}^{}}\left\| {{(\Theta _{*}^{{}}-\widetilde{\Theta }_{k,h}^{{}})}^{\top}}{{z}_{k,h}} \right\|_{2}^{2} \nonumber\\ 
 \le \ & {{D}^{}}\left\| {{(\Theta _{*}^{{}}-\widetilde{\Theta }_{k}^{{}})}^{\top}}\mathcal{V}_{h,k}^{1/2} \right\|_{2}^{2}\left\| \mathcal{V}_{h,k}^{-1/2}{{z}_{k,h}} \right\|_{2}^{2} \nonumber\\ 
 \le \ & {{D}^{}}\left\| \Theta _{*}^{{}}-\widetilde{\Theta }_{k}^{{}} \right\|_{\mathcal{V}_{h,k}^{{}}}^{2}\left\| \mathcal{V}_{h,k}^{-1/2}{{z}_{k,h}} \right\|_{2}^{2} \nonumber\\ 
 \le \ & {{D}^{}}{{\zeta }_{h}}^2(\delta )\left\| \mathcal{V}_{h,k}^{-1/2}{{z}_{k,h}} \right\|_{2}^{2} \,.
 \label{eq31}
\end{align}

Then, using (\ref{eq31}) and fact that ${{\left( {{\left\| \widetilde{P}_{k,h+1}^{1/2}\Theta _{*}^{\top}{{z}_{k,h}} \right\|}_{2}}-{{\left\| \widetilde{P}_{k,h+1}^{1/2}\widetilde{\Theta }_{k,h}^{\top}{{z}_{k,h}} \right\|}_{2}} \right)}^{2}}\le 2D(1+\gamma)^2$ we derive that

\begin{align}
  {{\left( {{\left\| \widetilde{P}_{k,h+1}^{1/2}\Theta _{*}^{\top}{{z}_{k,h}} \right\|}_{2}}-{{\left\| \widetilde{P}_{k,h+1}^{1/2}\widetilde{\Theta }_{k,h}^{\top}{{z}_{k,h}} \right\|}_{2}} \right)}^{2}} \le \ &2D{{\left( 1+\gamma  \right)}^{2}}\zeta _{h}^{2}(\delta )\min \left\{ \left\| \mathcal{V}_h^{-1/2}{{z}_{k,h}} \right\|_{2}^{2},1 \right\} \nonumber\\ 
 = \ &2D{{\left( 1+\gamma  \right)}^{2}}\zeta _{k}^{2}(\delta )\min \left\{ \left\| {{z}_{k,h}} \right\|_{\mathcal{V}_h^{-1}}^{2},1 \right\} \nonumber\\ 
 \le \ &2D{{\left( 1+\gamma  \right)}^{2}}\zeta _{h}^{2}(\delta )\ln \left( \left\| {{z}_{k,h}} \right\|_{{{\mathcal{V}_h^{-1}}}}^{2}+1 \right) \,.
 \label{eq32}
\end{align}

Therefore, we have
\begin{equation}
\begin{aligned}
  & \sum\limits_{k=1}^{K}{\sum\limits_{h=1}^{H}{{{\left( {{\left\| \widetilde{P}_{k,h+1}^{1/2}\Theta _{*}^{\top}{{z}_{k,h}} \right\|}_{2}}-{{\left\| \widetilde{P}_{k,h+1}^{1/2}\widetilde{\Theta }_{k,h}^{\top}{{z}_{k,h}} \right\|}_{2}} \right)}^{2}}}} 
 \le 4KHD{{\left( 1+\gamma  \right)}^{2}}\zeta _{k}^{2}(\delta )\ln \left( \frac{\det \left( \mathcal{V}_h\right)}{\det \left( \lambda I \right)} \right) \,,
 \label{eq33}
\end{aligned}
\end{equation}

where the last step in (\ref{eq33}) is derived applying Lemma \ref{logtrick} in Appendix \ref{ICMLAPPAX1}. Finally, taking the square root of (\ref{eq33}) and multiplying it by (\ref{ICMLTHIR}) yields to the final bound for the third term in Lemma \ref{ICMLL1}.
\end{proof}

\section{Proof of Theorem~\ref{ICMLT2222} in Section~\ref{section4}}
\label{PROOFNIPS}

\begin{proof}
The regret $\mathcal{R}(K)$ under the case $L\ge H$  can be decomposed as shown in~(\ref{eeeq23}). Recall that 

\begin{equation}
\begin{aligned}
  \mathcal{R}(K)\overset{(I)}{=}\ & \sum\limits_{k=1}^{K}{\sum\limits_{h=1}^{H}{{{I}_{1}}}}
  +  \sum\limits_{k=1}^{K}\sum\limits_{h=1}^{H}{{{\Delta }_{k,h}}} + \sum\limits_{k=1}^{K}{\sum\limits_{h=1}^{H}{||\Theta _{*}^{\top}{{z}_{k,h}}||_{{{\widetilde{P}}_{k,h+1}}}^{2}-}}||{{\widetilde{\Theta }}^{\top}}{{z}_{k,h}}||_{{{\widetilde{P}}_{k,h+1}}}^{2}\,, \nonumber\\
   \le \ & O\,\,\left( \sqrt{KH\ln \frac{1}{\delta }} \right)+O\left( \sqrt{K{{H}^{3}}\ln \frac{1}{\delta }} \right)+\underbrace{\ell_A \times \ell_B}_{\mathcal{R}'''(K)}\,.
\end{aligned}
\end{equation}

Where the first two terms in $(I)$ enjoy the same regret bound under the Corollaries and Propositions in Section~\ref{ICMLT4.5}. And the third term in $(I)$ could be decomposed as $\mathcal{R}'''(K)=\ell_A\times \ell_B$ according to (\ref{nips25}). The only difference is the regret boud w.r.t $\ell_B$. To simplify the notations for the later proof, we denote $\ell_B$ as 
\[\mathcal{\ell_B}= {{\left[ \sum\limits_{k=1}^{K}{\sum\limits_{h=1}^{H}{{{\left( ||\widetilde{P}_{k,h+1}^{1/2}\Theta _{*}^{\top}{{z}_{k,h}}|{{|}_{2}}-||\widetilde{P}_{k,h+1}^{1/2}{{\widetilde{\Theta }}^{\top}}{{z}_{k,h}}|{{|}_{2}} \right)}^{2}}}} \right]}^{1/2}}={{\left[  \mathcal{R}_{\ell_B}(K)\right]}^{1/2}}\,\]

Recall (\ref{eq31}), the following holds
\begin{align}
  {{\left( {{\left\| \widetilde{P}_{k,h+1}^{1/2}\Theta _{*}^{\top}{{z}_{k,h}} \right\|}_{2}}-{{\left\| \widetilde{P}_{k,h+1}^{1/2}\widetilde{\Theta }_{k,h}^{\top}{{z}_{k,h}} \right\|}_{2}} \right)}^{2}} 
 \le  {{D}^{}}{{\zeta }_{h}}^2(\delta )\left\| \mathcal{V}_{h,k}^{-1/2}{{z}_{k,h}} \right\|_{2}^{2} \,. \nonumber
\end{align}

Hence, dynamic regret $\mathcal{R}(\pounds)$ within the epoch $\pounds$ is bounded by

\begin{align}
\mathcal{R}(\pounds)=\ &\sum\limits_{h\in \pounds} \nonumber 2D{{\left( 1+\gamma  \right)}^{2}}\zeta _{h}^{2}(\delta )\left\| \mathcal{V}_{h,k}^{-1/2}{{z}_{k,h}} \right\|\\ 
 \le \ &\sum\limits_{h\in \pounds} \nonumber 2D{{\left( 1+\gamma  \right)}^{2}}({\ell_1+\ell_2+\ell_3})^2\left\| \mathcal{V}_{h,k}^{-1/2}{{z}_{k,h}} \right\|\nonumber\\ 
\le\ & \sum\limits_{h\in \pounds}   \Upsilon_1L\mathcal{B}_L^2 +\Upsilon_2\sqrt{L}\mathcal{B}_L+\Upsilon_3 \nonumber\\ 
\le\ & L (\Upsilon_1L\mathcal{B}_L^2 +\Upsilon_2\sqrt{L}\mathcal{B}_L+\Upsilon_3) \nonumber\\
 =\ &\underbrace{L^2\mathcal{B}_L^2\Upsilon_1}_{\ell^{'}}+\underbrace{L^{3/2}\mathcal{B}_L\Upsilon_2}_{\ell^{''}}+\underbrace{L\Upsilon_3}_{\ell^{'''}}
\label{eq71}
\end{align}


where ${B}_L=\sum\limits_{p={{h}_{0}}}^{h-1}{{{\left\| {{\Theta }_{p}}-{{\Theta }_{p+1}} \right\|}_{F}}}$ is the total variability in one epoch $\mathcal{L}$, $\Upsilon_1=\tfrac{(m+n)}{\lambda }\Upsilon_4$, $\Upsilon_2=\left( \sqrt{\lambda}+ \upsilon_w \sqrt{2\ln \left( \tfrac{1}{\delta } \right)+n\ln \tfrac{\det ({\mathcal{V}_{h}})}{\det (\lambda I)}}\right)\Upsilon_4$, $\Upsilon_3=\Upsilon_4\left( {\lambda}+ \upsilon_w^2 \left({2\ln \left( \tfrac{1}{\delta } \right)+n\ln \tfrac{\det ({\mathcal{V}_{h}})}{\det (\lambda I)}}\right)+\sqrt{\lambda}\upsilon_w \sqrt{2\ln \left( \tfrac{1}{\delta } \right)+n\ln \tfrac{\det ({\mathcal{V}_{h}})}{\det (\lambda I)}}\right)$ and $\Upsilon_4=2D{{\left( 1+\gamma  \right)}^{2}}\left\| \mathcal{V}_{h,k}^{-1/2}{{z}_{k,h}} \right\|_{2}^{2}$.

Then, we can bound ${\mathcal{\ell_B}}$ term  $\ell^{'}$, $\ell^{''}$ and $\ell^{''}$ using Lemma~\ref{logtrick} as

$\ell^{'}\le \underbrace{\frac{4(m+n)^2( 1+\gamma  )^{2}D}{\lambda}\log \left( 1+\frac{{{(1+\gamma)}^{2}}L}{\lambda (m+n)} \right)}_{\chi'}\mathcal{B}_{L}^2L$

$\ell^{''}\le \underbrace{{4(m+n)( 1+\gamma  )^{2}}D\left( \sqrt{\lambda}+ \upsilon_w \sqrt{2\ln \left( \tfrac{1}{\delta } \right)+n\ln \tfrac{\det ({\mathcal{V}_{h}})}{\det (\lambda I)}}\right)\log \left( 1+\frac{{{(1+\gamma)}^{2}}L}{\lambda (m+n)} \right)}_{\chi''}\mathcal{B}_{L}L^{1/2}$

$\ell^{'''}\le \underbrace{{4(m+n)( 1+\gamma  )^{2}}D\left( {\lambda}+ \upsilon_w^2 \left({2\ln \left( \tfrac{1}{\delta } \right)+n\ln \tfrac{\det ({\mathcal{V}_{h}})}{\det (\lambda I)}}\right)+\sqrt{\lambda}\upsilon_w \sqrt{2\ln \left( \tfrac{1}{\delta } \right)+n\ln \tfrac{\det ({\mathcal{V}_{h}})}{\det (\lambda I)}}\right)\log \left( 1+\frac{{{(1+\gamma)}^{2}}L}{\lambda (m+n)} \right)}_{\chi'''}$

By taking the union bound over the dynamic regret of all $ \lceil HK/L \rceil$ epochs, we know that the following holds with probability at least $1-2/HK$

\begin{align}
\mathcal{R}_{\ell_B}(K)=\ &\sum\limits_{s=1}^{ \lceil HK/L \rceil} \nonumber \mathcal{R}(\pounds_s)\\ 
 \le \ & \frac{HK}{L}   (\ell^{'}+\ell^{''}+\ell^{'''})\nonumber\\
 =\ & \frac{L^2}{HK}\frac{(HK)^2}{L^2} \mathcal{B}_{L}^2\chi'+L^{1/2}\frac{HK}{L} \mathcal{B}_{L}\chi''  +\frac{HK}{L} \chi'''\nonumber\\
  =\ & \frac{L^2}{HK} \mathcal{B}_{HK}^2\chi'+L^{1/2} \mathcal{B}_{HK}\chi''  +\frac{HK}{L} \chi'''
\label{eq71111}
\end{align}

Putting $\mathcal{R}_{\ell_B}(K)$ to $\ell_B$, we obtain the bound for $\mathcal{R}'''(K)$ as
\begin{align}
  \mathcal{R}'''(K)=\ell_A \times \ell_B= \ & 2D\sqrt{KH}(1+\gamma )\sqrt{
\frac{L^2}{HK} \mathcal{B}_{HK}^2\chi'+L^{1/2} \mathcal{B}_{HK}\chi''  +\frac{HK}{L} \chi'''} \nonumber \\
\le \ & 2D\sqrt{KH}(1+\gamma ) \left(
\sqrt{\frac{L^2}{HK} \mathcal{B}_{HK}^2\chi'}+\sqrt{L^{1/2} \mathcal{B}_{HK}\chi''}  +\sqrt{\frac{HK}{L} \chi'''} \right) \,.
\end{align}

Ignoring logarithmic factors, we finally obtain that 

\begin{align*}
    \mathcal{R}(K)
    \le \ & \widetilde{O}\left(L\mathcal{B}_{HK}+HK\sqrt{\frac{1}{L}} +\sqrt{HK}\sqrt{\mathcal{B}_{HK}}L^{1/4} \right)+\widetilde{O}\,\,\left( \sqrt{KH} \right)+\widetilde{O}\left( \sqrt{K{{H}^{3}}} \right) \,
\end{align*}

\end{proof}

\section{Auxiliary Proof and Lemma }
\label{ICMLAPPAX1}

In this section, we provide several technical lemmas frequently used in the proofs.

\begin{lemma}
 (Azuma-Hoeffding inequality)
Let $\,\{{{X}_{k}}\}_{k=0}^{\infty }$ be a discrete-parameter real-valued
martingale sequence such that for every $k\in \mathbb{N}$, the condition $|X_k-X_{k-1}|\le \mu$ holds for some
non-negative constant $\mu$. Then with probability at least $1-\delta$, we have

\[\left| {{X}_{n}}-{{X}_{0}} \right|\le 2\mu \sqrt{n\log \frac{1}{\delta }} \,.\]

\label{Azu-Hoe}
\end{lemma}

\begin{lemma} \cite{abbasi2011improved}
For any $\{{{x}_{t}}\}_{t=1}^{T}\in {\mathbb{R}^{d}}$ satisfying that ${{\left\| {{x}_{t}} \right\|}_{2}}\le L$, let ${{A}_{0}}=\lambda I$ and ${{A}_{t}}={{A}_{0}}+\sum\limits_{i=1}^{t-1}{{{x}_{i}}x_{i}^{T}}$, then the following inequality holds 

\[\sum\limits_{t=1}^{T}{\min {{\{1,\,{{\left\| {{x}_{t}} \right\|}_{A_{t-1}^{-1}}}\}}^{2}}}\le 2d\log \frac{d\lambda +T{{L}^{2}}}{d\lambda } \,.\]

\begin{proof}
Notice that the following holds
\[{{A}_{t}}={{A}_{t-1}}+{{x}_{t}}x_{t}^{\top}=A_{t-1}^{1/2}\left( I+A_{t-1}^{-1/2}{{x}_{t}}x_{t}^{\top}A_{t-1}^{-1/2} \right)A_{t-1}^{1/2}  \,.\]

and taking the determinant yields
\[\det \left( {{A}_{t}} \right)=\det \left( {{A}_{t-1}} \right)\det \left( I+A_{t-1}^{-1/2}{{x}_{t}}x_{t}^{\top}A_{t-1}^{-1/2} \right)\,.\]

Note the fact $\det(I+xx^{\top})=1+||x||_2^2 $, we have
\[\det \left( {{A}_{t}} \right)=\det \left( {{A}_{t-1}} \right)\left( 1+\left\| A_{t-1}^{-1/2}{{x}_{t}} \right\|_{2}^{2} \right) \ge \det \left( {{A}_{t-1}} \right)\exp \left( \frac{\left\| A_{t-1}^{-1/2}{{x}_{t}} \right\|_{2}^{2}}{2} \right) 
\,.\]

where the inequality holds based on fact $1+x\ge \exp(x/2)$ holds for $x\in [0,1]$. Finally, by utilizing telescope structure, we get
\[\sum\limits_{t=1}^{T}{\left\| A_{t-1}^{-1/2}{{x}_{t}} \right\|}_{2}^{2}\le 2\log \frac{\det ({{A}_{T}})}{\det ({{A}_{0}})}\le 2d\log \left( 1+\frac{{{L}^{2}}T}{\lambda d} \right)\,.\]

\end{proof}

\label{logtrick}
\end{lemma}

\subsection{Proof of Lemma \ref{ICMLTV}}
\label{ICMLAPPAX11}

In this section, we provide the proof of Lemma \ref{ICMLTV}. 

\begin{proof}
For any $h_0 \in [H]$, one has

\begin{align}
  &{{\left\| {{\left( \sum\limits_{s={{h}_{0}}}^{h-1}{Z_{s}^{\top}{{Z}_{s}}+\lambda I} \right)}^{-1}}\left( \sum\limits_{s={{h}_{0}}}^{h-1}{Z_{s}^{\top}{{Z}_{s}}(\Theta _{s}^{*}-\Theta _{h}^{*})} \right) \right\|}_{F}} \nonumber\\ 
 = \ &{{\left\| {{\left( \sum\limits_{s={{h}_{0}}}^{h-1}{Z_{s}^{\top}{{Z}_{s}}+\lambda I} \right)}^{-1}}\left( \sum\limits_{s={{h}_{0}}}^{h-1}{Z_{s}^{\top}{{Z}_{s}}\left( \sum\limits_{p=s}^{h-1}{({{\Theta }_{p}}-{{\Theta }_{p+1}})} \right)} \right) \right\|}_{F}} \nonumber\\ 
 = \ &{{\left\| {{\left( \sum\limits_{s={{h}_{0}}}^{h-1}{Z_{s}^{\top}{{Z}_{s}}+\lambda I} \right)}^{-1}}\left( \sum\limits_{p={{h}_{0}}}^{h-1}{\left( \sum\limits_{s={{h}_{0}}}^{p}{Z_{s}^{\top}{{Z}_{s}}({{\Theta }_{p}}-{{\Theta }_{p+1}})} \right)} \right) \right\|}_{F}} \nonumber\\ 
 \le \ & \sum\limits_{p={{h}_{0}}}^{h-1}{{{\left\| {{\left( \sum\limits_{s={{h}_{0}}}^{h-1}{Z_{s}^{\top}{{Z}_{s}}+\lambda I} \right)}^{-1}}\left( \sum\limits_{s={{h}_{0}}}^{p}{Z_{s}^{\top}{{Z}_{s}}} \right)({{\Theta }_{p}}-{{\Theta }_{p+1}}) \right\|}_{F}}} \nonumber \\ 
 \le \ &\sum\limits_{p={{h}_{0}}}^{h-1}{{{\left\| {{\left( \sum\limits_{s={{h}_{0}}}^{h-1}{Z_{s}^{\top}{{Z}_{s}}+\lambda I} \right)}^{-1}}\left( \sum\limits_{s={{h}_{0}}}^{p}{Z_{s}^{\top}{{Z}_{s}}} \right)({{\Theta }_{p}}-{{\Theta }_{p+1}}) \right\|}_{F}}} \nonumber\\ 
 \le \ & \sum\limits_{p={{h}_{0}}}^{h-1}{{{\left\| {{\left( \sum\limits_{s={{h}_{0}}}^{h-1}{Z_{s}^{\top}{{Z}_{s}}+\lambda I} \right)}^{-1}}\sum\limits_{s={{h}_{0}}}^{p}{Z_{s}^{\top}{{Z}_{s}}} \right\|}_{2}}{{\left\| {{\Theta }_{p}}-{{\Theta }_{p+1}} \right\|}_{F}}} \,,
\label{equ10}
\end{align}

\noindent
where the last inequality holds due to fact that ${{\left\| AB \right\|}_{F}}\le {{\left\| A \right\|}_{2}}{{\left\| B \right\|}_{F}}$ for any matrix $A$ and $B$.
Since ${\left\| v \right\|}_{{{\left( \sum\limits_{s={{h}_{0}}}^{h-1}{Z_{s}^{\top}{{Z}_{s}}+\lambda I} \right)}^{-1}}}\le \frac{{{\left\| v \right\|}_{2}}}{\sqrt{\lambda }}$ as ${\left( \sum\limits_{s=h_0}^{h-1}{Z_{s}^{\top}{{Z}_{s}}+\lambda I} \right)}\succeq \lambda I$ holds, thus we have

\begin{align}
 {{\left\| {{\left( \sum\limits_{s={{h}_{0}}}^{h-1}{Z_{s}^{\top}{{Z}_{s}}+\lambda I} \right)}^{-1}}\sum\limits_{s={{h}_{0}}}^{p}{Z_{s}^{\top}{{Z}_{s}}} \right\|}_{2}} 
 = \ &\underset{v\in \mathcal{B}(1)}{\mathop{\sup }}\,\left| {{v}^{\top}}{{\left( \sum\limits_{s={{h}_{0}}}^{h-1}{Z_{s}^{\top}{{Z}_{s}}+\lambda I} \right)}^{-1}}\sum\limits_{s={{h}_{0}}}^{p}{Z_{s}^{\top}{{Z}_{s}}v} \right| \nonumber\\ 
 \overset{(a)}{\le}\ & \left| v_{*}^{\top}{{\left( \sum\limits_{s={{h}_{0}}}^{h-1}{Z_{s}^{T}{{Z}_{s}}+\lambda I} \right)}^{-1}}\sum\limits_{s={{h}_{0}}}^{p}{Z_{s}^{\top}{{Z}_{s}}{{v}_{*}}} \right| \nonumber\\ 
 \le \ & {{\left\| {{v}_{*}} \right\|}_{{{\left( \sum\limits_{s={{h}_{0}}}^{h-1}{Z_{s}^{\top}{{Z}_{s}}+\lambda I} \right)}^{-1}}}}{{\left\| \sum\limits_{s={{h}_{0}}}^{p}{Z_{s}^{\top}{{Z}_{s}}{{v}_{*}}} \right\|}_{{{\left( \sum\limits_{s={{h}_{0}}}^{h-1}{Z_{s}^{\top}{{Z}_{s}}+\lambda I} \right)}^{-1}}}}\nonumber \\ 
 \le \ & {{\left\| {{v}_{*}} \right\|}_{{{\left( \sum\limits_{s={{h}_{0}}}^{h-1}{Z_{s}^{\top}{{Z}_{s}}+\lambda I} \right)}^{-1}}}}{{\left\| \sum\limits_{s={{h}_{0}}}^{p}{Z_{s}^{\top}{{\left\| {{Z}_{s}} \right\|}_{2}}{{\left\| {{v}_{*}} \right\|}_{2}}} \right\|}_{{{\left( \sum\limits_{s={{h}_{0}}}^{h-1}{Z_{s}^{\top}{{Z}_{s}}+\lambda I} \right)}^{-1}}}} \nonumber\\ 
 \le \ & \frac{1+\gamma}{\sqrt{\lambda }}{{\left\| \sum\limits_{s={{h}_{0}}}^{p}{{{Z}_{s}}} \right\|}_{{{\left( \sum\limits_{s={{h}_{0}}}^{h-1}{Z_{s}^{\top}{{Z}_{s}}+\lambda I} \right)}^{-1}}}} \nonumber\\ 
 \le \ & \frac{1+\gamma}{\sqrt{\lambda }}\sum\limits_{s={{h}_{0}}}^{p}{{{\left\| {{Z}_{s}} \right\|}_{{{\left( \sum\limits_{s={{h}_{0}}}^{h-1}{Z_{s}^{\top}{{Z}_{s}}+\lambda I} \right)}^{-1}}}}} \nonumber\\ 
 \overset{(b)}{\le}\ & \frac{1+\gamma}{\sqrt{\lambda }}\sqrt{L}\sqrt{\sum\limits_{s={{h}_{0}}}^{p}{\left\| {{Z}_{s}} \right\|_{{{\left( \sum\limits_{s={{h}_{0}}}^{h-1}{Z_{s}^{\top}{{Z}_{s}}+\lambda I} \right)}^{-1}}}^{2}}} \nonumber\\ 
 \overset{(c)}{\le} \ & (1+\gamma)\sqrt{\frac{{L(m+n)}}{\lambda }} \,,
 \label{PF1}
\end{align}

where $v_*$ in $(a)$ denotes the optimizer; $(b)$ holds by Cauchy-Schwarz inequality. The inequality $(c)$ makes use of the following algebra formulation: for $p\in \{{{h}_{0}},...,h-1\}$

\begin{align}
  & \sum\limits_{s={{h}_{0}}}^{p}{{{\left\| {{Z}_{s}} \right\|}^2_{{{\left( \sum\limits_{s={{h}_{0}}}^{h-1}{Z_{s}^{\top}{{Z}_{s}}+\lambda I} \right)}^{-1}}}}} \nonumber\\ 
 = \ & tr\left[ {{Z}_{s}}{{\left( \sum\limits_{s={{h}_{0}}}^{h-1}{Z_{s}^{\top}{{Z}_{s}}+\lambda I} \right)}^{-1}}Z_{s}^{\top} \right] \nonumber\\ 
 = \ &tr\left[ {{\left( \sum\limits_{s={{h}_{0}}}^{h-1}{Z_{s}^{\top}{{Z}_{s}}+\lambda I} \right)}^{-1}}\left( \sum\limits_{s={{h}_{0}}}^{p}{Z_{s}^{\top}{{Z}_{s}}} \right) \right] \nonumber\\ 
 \le \ & tr\left[ {{\left( \sum\limits_{s={{h}_{0}}}^{h-1}{Z_{s}^{\top}{{Z}_{s}}+\lambda I} \right)}^{-1}}\left( \sum\limits_{s={{h}_{0}}}^{p}{Z_{s}^{\top}{{Z}_{s}}} \right) \right] \nonumber \\ 
 & \ +\sum\limits_{s=p+1}^{h-1}{Z_{s}^{\top}{{\left( \sum\limits_{s={{h}_{0}}}^{h-1}{Z_{s}^{\top}{{Z}_{s}}+\lambda I} \right)}^{-1}}{{Z}_{s}}}+\lambda \sum\limits_{i=1}^{d}{e_{i}^{\top}{{\left( \sum\limits_{s={{h}_{0}}}^{h-1}{Z_{s}^{\top}{{Z}_{s}}+\lambda I} \right)}^{-1}}{{e}_{i}}} \nonumber\\ 
 = \ &tr\left[ {{\left( \sum\limits_{s={{h}_{0}}}^{h-1}{Z_{s}^{\top}{{Z}_{s}}+\lambda I} \right)}^{-1}}\left( \sum\limits_{s={{h}_{0}}}^{p}{Z_{s}^{\top}{{Z}_{s}}} \right) \right]  +tr\left[ {{\left( \sum\limits_{s={{h}_{0}}}^{h-1}{Z_{s}^{\top}{{Z}_{s}}+\lambda I} \right)}^{-1}}\left( \sum\limits_{s=p=1}^{h-1}{Z_{s}^{\top}{{Z}_{s}}} \right) \right] \nonumber \\
 & \ +tr\left[ {{\left( \sum\limits_{s={{h}_{0}}}^{h-1}{Z_{s}^{\top}{{Z}_{s}}+\lambda I} \right)}^{-1}}\lambda \sum\limits_{i=1}^{d}{e_{i}^{\top}{{e}_{i}}} \right] \nonumber\\ 
 = \ &tr({{I}_{n+m}})=n+m \,.
\end{align}

Finally, putting Assumption \ref{A3} and (\ref{PF1}) into (\ref{equ10}) finishes our proof.
\end{proof}


\end{document}